\documentclass[11pt]{article}
\usepackage[title]{appendix}
\newif\ifarxiv
\arxivtrue
\usepackage{adjustbox} 

\usepackage[margin=1in]{geometry}
\usepackage[noblocks]{authblk}
\usepackage{microtype}
\usepackage{graphicx}
\usepackage{booktabs} \usepackage{subfigure}

\usepackage{multirow}
\usepackage{amssymb,amsmath,dsfont,multirow}
\usepackage{amsthm}
\usepackage{thmtools}
\usepackage{tikz}

\usepackage{hyperref}

\usepackage{natbib}
\usepackage{algorithm}
\usepackage{algorithmic}

\numberwithin{equation}{section}
\declaretheorem[numberlike=equation]{theorem}

\declaretheorem[numberlike=theorem]{proposition}

\declaretheoremstyle[qed={\ensuremath\Diamond}]{remstyle}

\usepackage[inline]{enumitem}
\usepackage{cleveref}
\usepackage{commath}
\usepackage{xspace}
\usepackage{nicefrac}
\def\R{\mathbb R}
\def\1{\mathds 1}

\def\eps{\epsilon}

\def\hnu{\hat \nu}
\def\hpsi{\hat\psi}

\DeclareMathOperator*{\argmin}{arg\,min}

\crefname{algocf}{alg.}{algs.}
\Crefname{algocf}{Algorithm}{Algorithms}

\def\ddefloop#1{\ifx\ddefloop#1\else\ddef{#1}\expandafter\ddefloop\fi}
\def\ddef#1{\expandafter\def\csname bb#1\endcsname{\ensuremath{\mathbb{#1}}}}
\ddefloop ABCDEFGHIJKLMNOPQRSTUVWXYZ\ddefloop
\def\ddef#1{\expandafter\def\csname c#1\endcsname{\ensuremath{\mathcal{#1}}}}
\ddefloop ABCDEFGHIJKLMNOPQRSTUVWXYZ\ddefloop
\def\ddef#1{\expandafter\def\csname h#1\endcsname{\ensuremath{\hat{#1}}}}
\ddefloop ABCDEFGHIJKLMNOPQRSTUVWXYZ\ddefloop

\def\hnu{\hat\nu}

\title{A Gradual, Semi-Discrete Approach to Generative Network Training
	via Explicit Wasserstein Minimization}

\author[1]{Yucheng Chen}
\author[1]{Matus Telgarsky}
\author[1]{Chao Zhang}
\author[1]{Bolton Bailey}
\author[2]{\\Daniel Hsu}
\author[1]{Jian Peng}
\affil[1]{University of Illinois at Urbana-Champaign, Urbana, IL}
\affil[2]{Columbia University, New York, NY}

\date{}

\begin{document}
\maketitle

\begin{abstract}
This paper provides a simple procedure to fit generative networks to target distributions,
with the goal of a small Wasserstein distance (or other optimal transport costs).
The approach is based on two principles:
(a) if the source randomness of the network is a continuous distribution (the ``semi-discrete'' setting),
then the Wasserstein distance is realized by a deterministic optimal transport mapping;
(b) given an optimal transport mapping between a generator network and a target distribution,
the Wasserstein distance may be decreased via a regression between the generated data
and the mapped target points.
The procedure here therefore alternates these two steps, forming an optimal transport and regressing against
it, gradually adjusting the generator network towards the target distribution.
Mathematically, this approach is shown to minimize the Wasserstein distance to both the empirical
target distribution, and also its underlying population counterpart.
Empirically, good performance is demonstrated on the training and testing sets of the MNIST and Thin-8 data.
The paper closes with a discussion of the unsuitability of the Wasserstein distance for certain
tasks, as has been identified in prior work \citep{generalization_equilibrium, thin8}.
\end{abstract}

\section{Introduction}

A generative network $g$ models a distribution by first sampling $x\sim\mu$ from
some simple distribution $\mu$ (e.g., a multivariate Gaussian),
and thereafter outputting $g(x)$; this sampling procedure defines a \emph{pushforward distribution} $g\#\mu$.
A common training procedure to fit $g\#\mu$
to a target distribution $\hnu$ is
to minimize a divergence $\cD(g\#\mu,\hnu)$ over a collection of parameters
defining $g$.

The original algorithms for this framework,
named \emph{generative adversarial networks},
alternatively optimized both the generator network $g$,
as well as a 
a second \emph{discriminator} of \emph{adversarial} network:
first the discriminator was fixed and the generator was optimized to fool it,
and second the generator was fixed and the discriminator was optimized to distinguish
it from $\hnu$.
This procedure was originally constructed to minimize a Jensen-Shannon Divergence
via a game-theoretic derivation \citep{gan}.
Subsequent work derived the adversarial relationship in other ways, for instance
the Wasserstein GAN used duality properties of the Wasserstein distance \citep{WGAN}.

\begin{figure}[t!]
	\centering
\ifarxiv
            \begin{adjustbox}{width=0.7\textwidth}
          \else
          \fi
	\begin{tikzpicture}
	\tikzset{node style/.style={midway,above,sloped}};
	\node[inner sep=0pt] (f1) at (0,0)
	{\includegraphics[width=.09\textwidth]{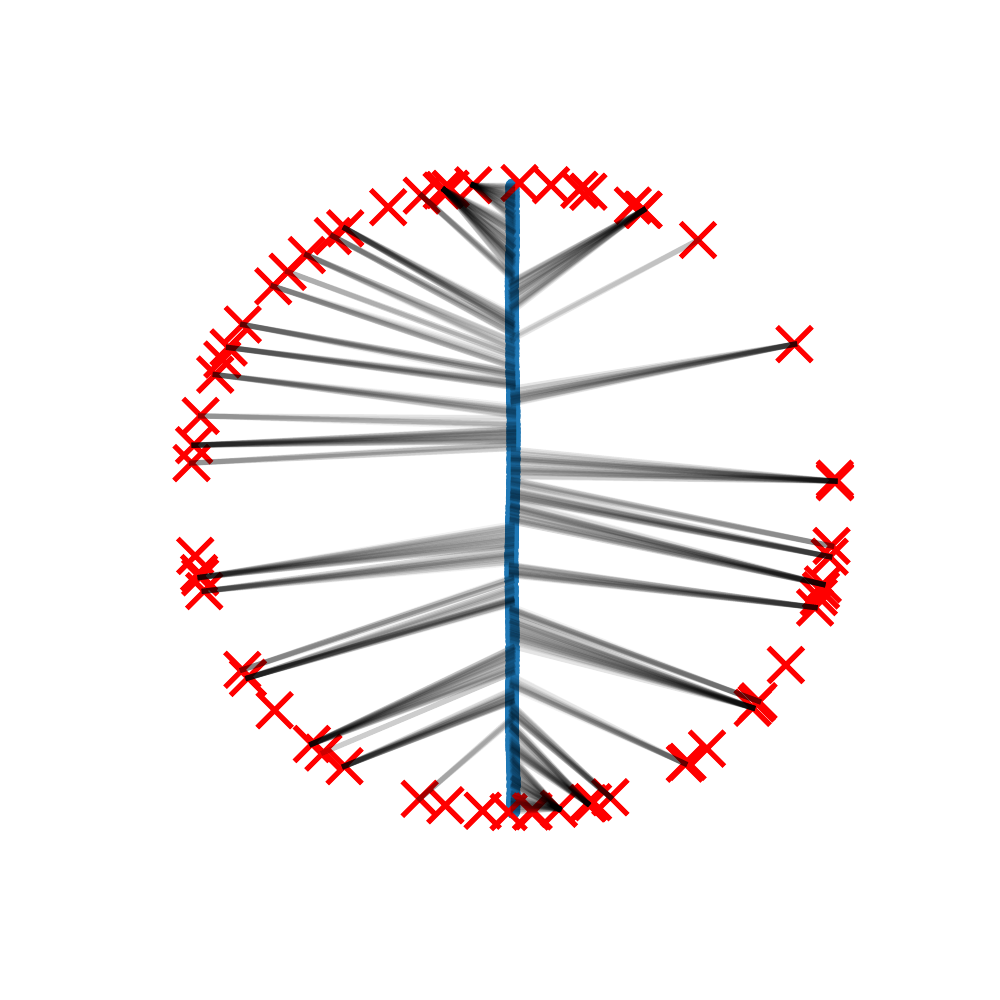}};
	\node[inner sep=0pt] (f2) at (2.1,0)
	{\includegraphics[width=.09\textwidth]{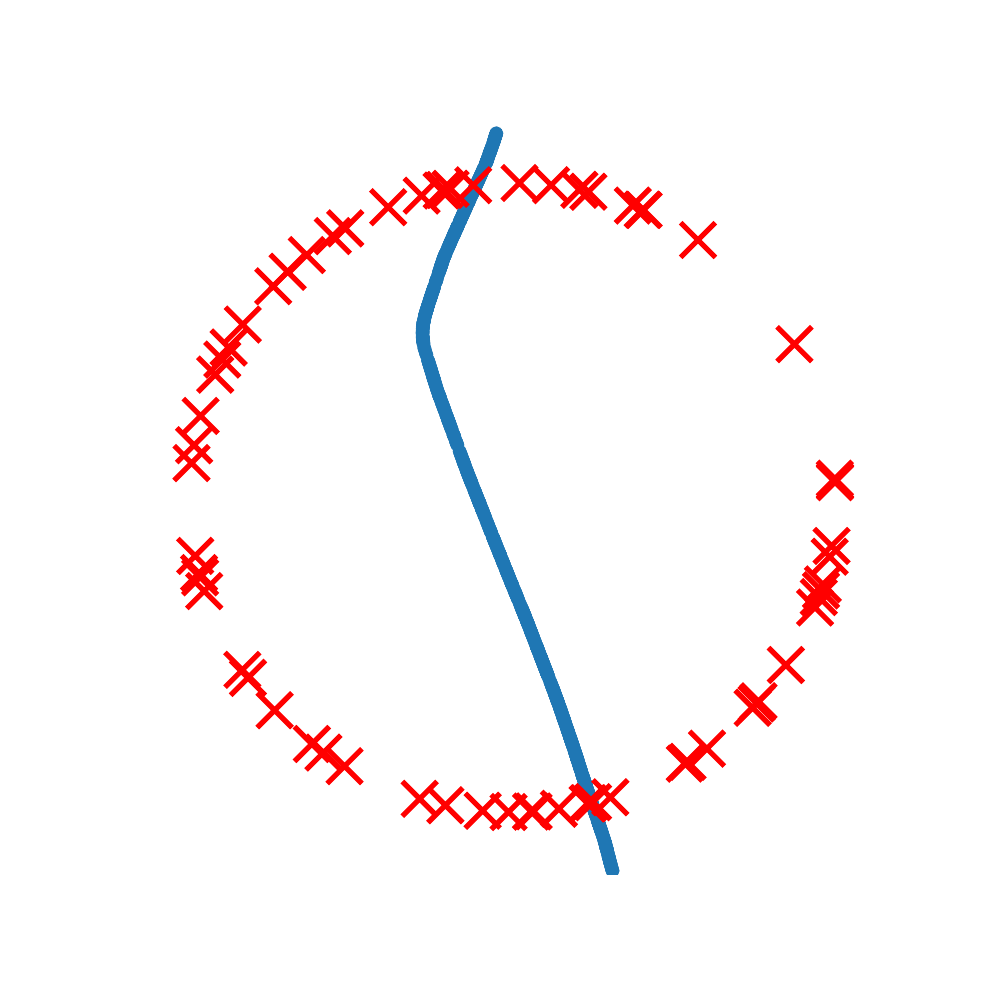}};
	\node[inner sep=0pt] (f3) at (0,-2.1)
	{\includegraphics[width=.09\textwidth]{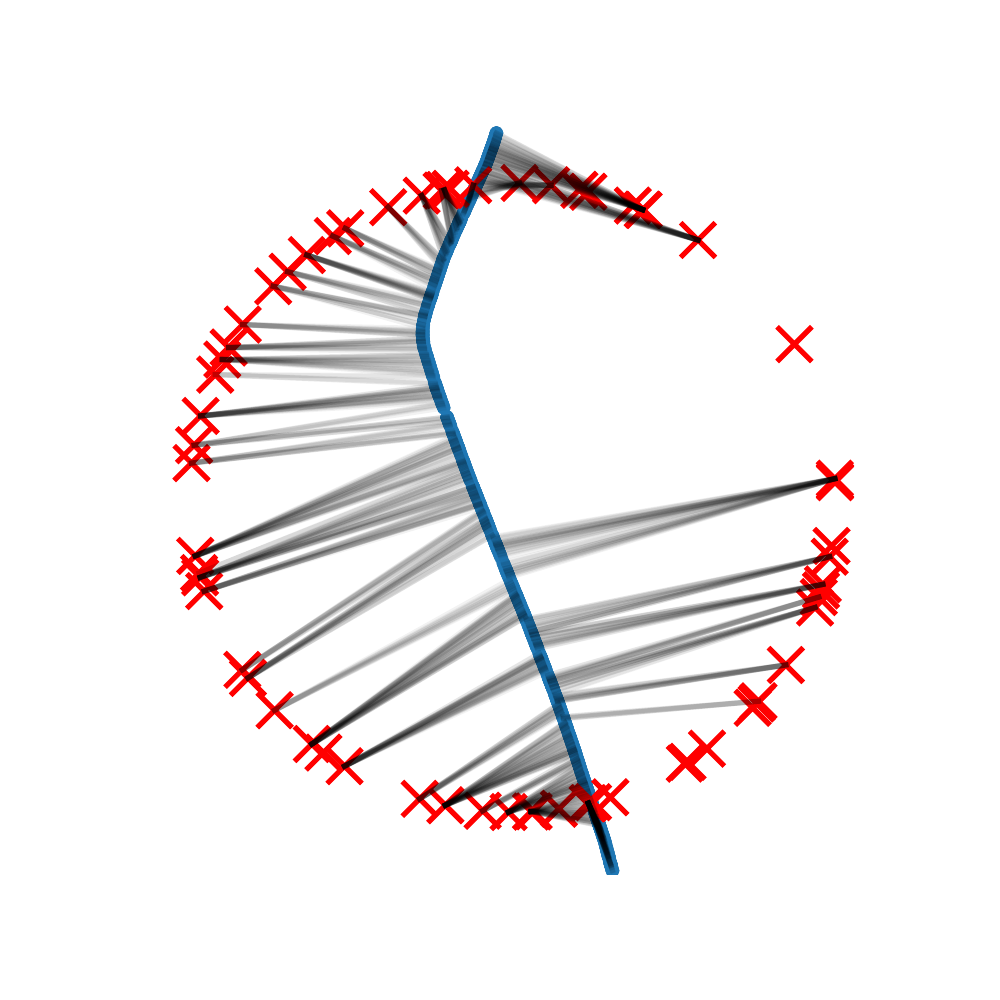}};
	\node[inner sep=0pt] (f4) at (2.1,-2.1)
	{\includegraphics[width=.09\textwidth]{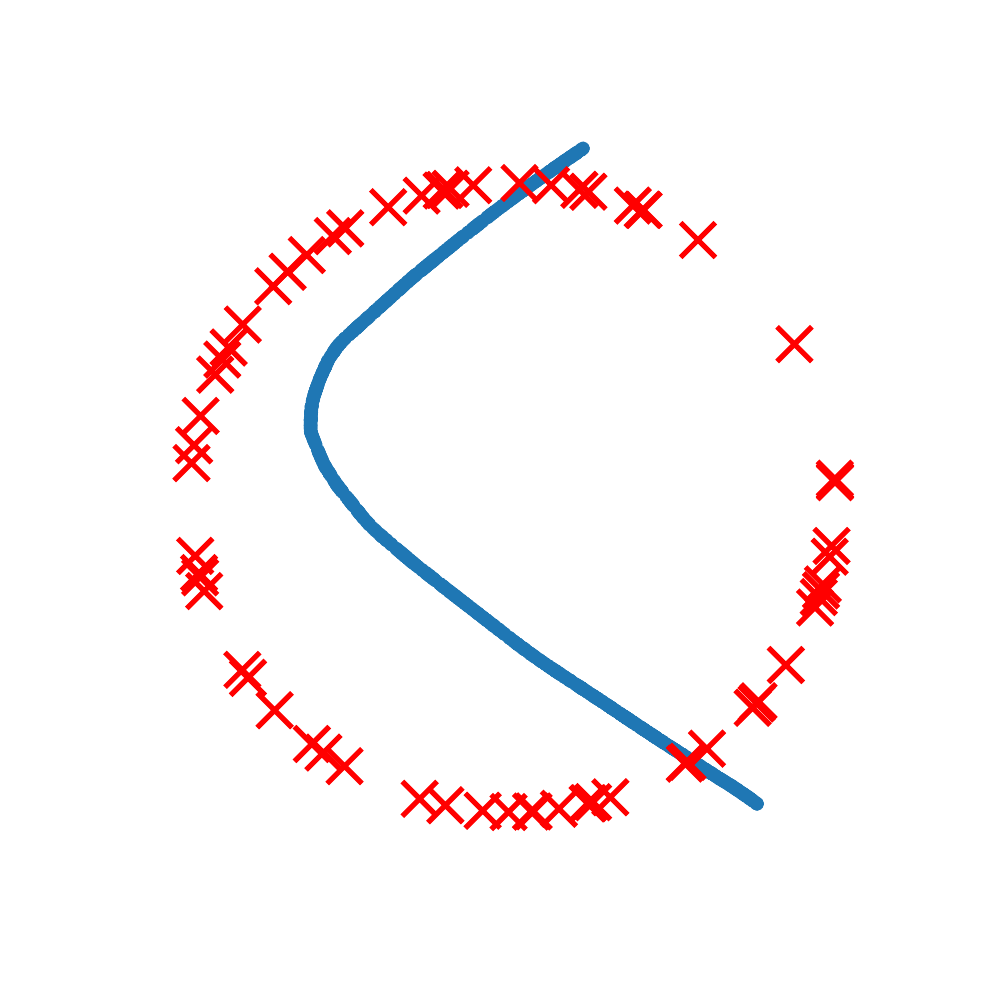}};
	\node[inner sep=0pt] (f5) at (4.2,0)
	{\includegraphics[width=.09\textwidth]{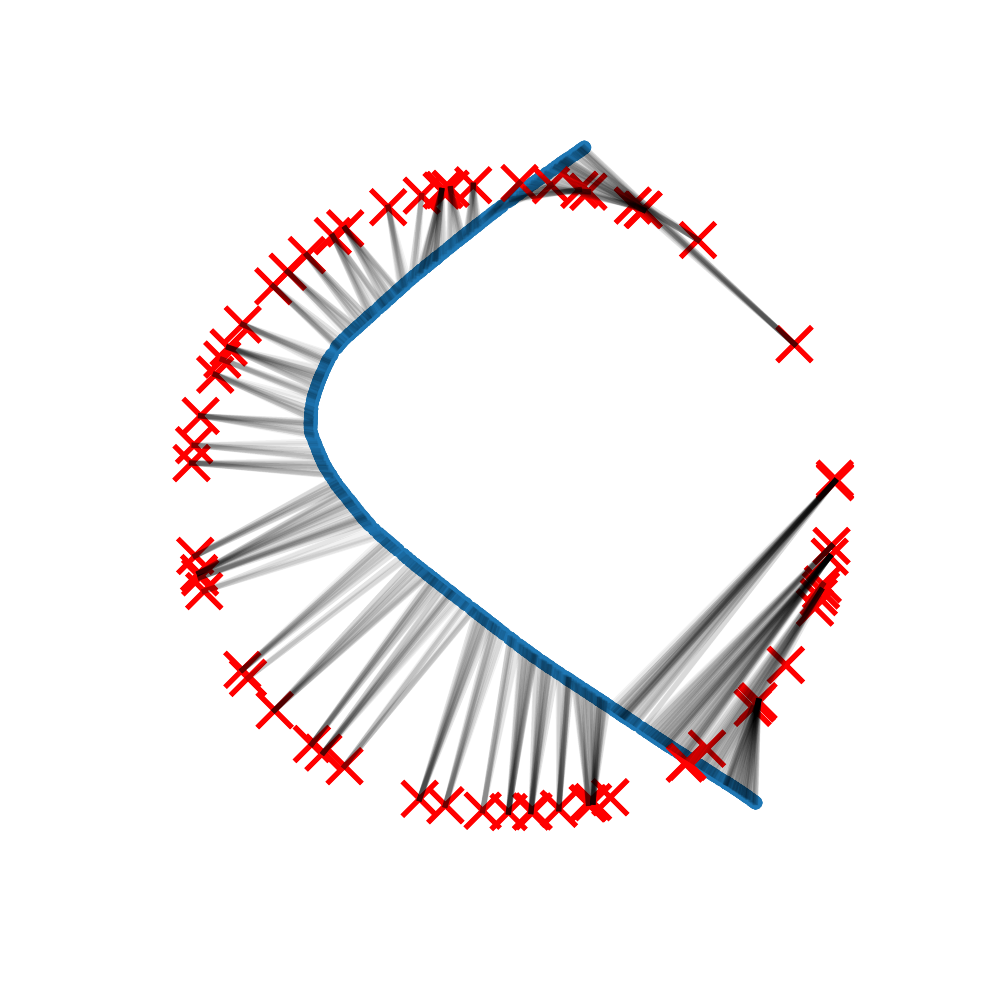}};
	\node[inner sep=0pt] (f6) at (6.3,0)
	{\includegraphics[width=.09\textwidth]{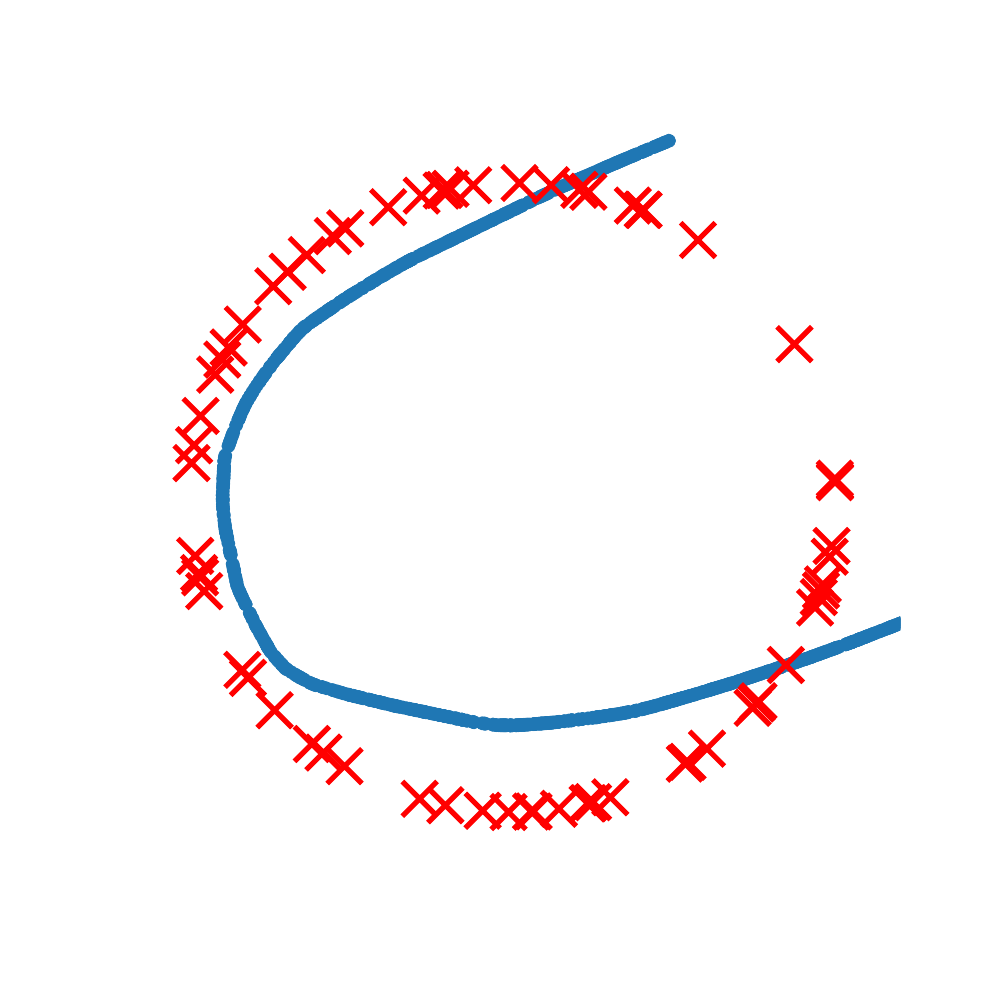}};
	\node[inner sep=0pt] (f7) at (4.2,-2.1)
	{\includegraphics[width=.09\textwidth]{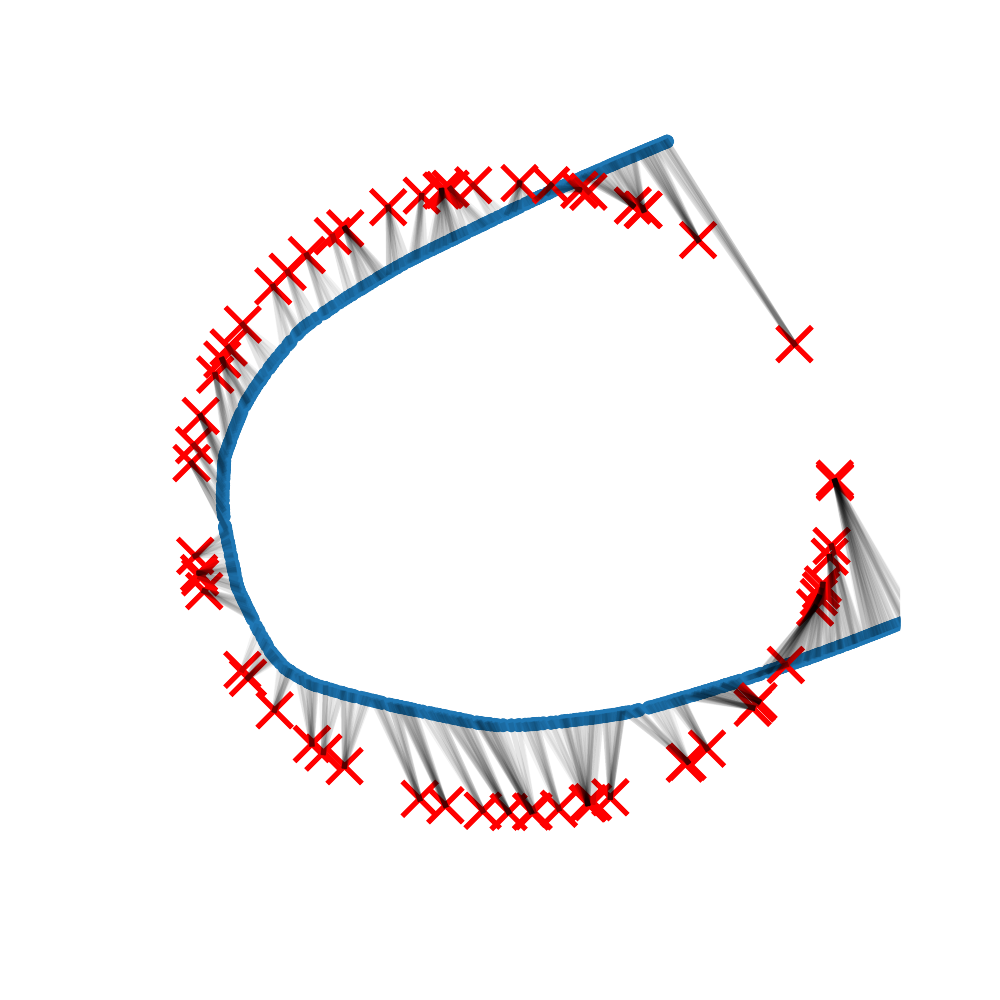}};
	\node[inner sep=0pt] (f8) at (6.3,-2.1)
	{\includegraphics[width=.09\textwidth]{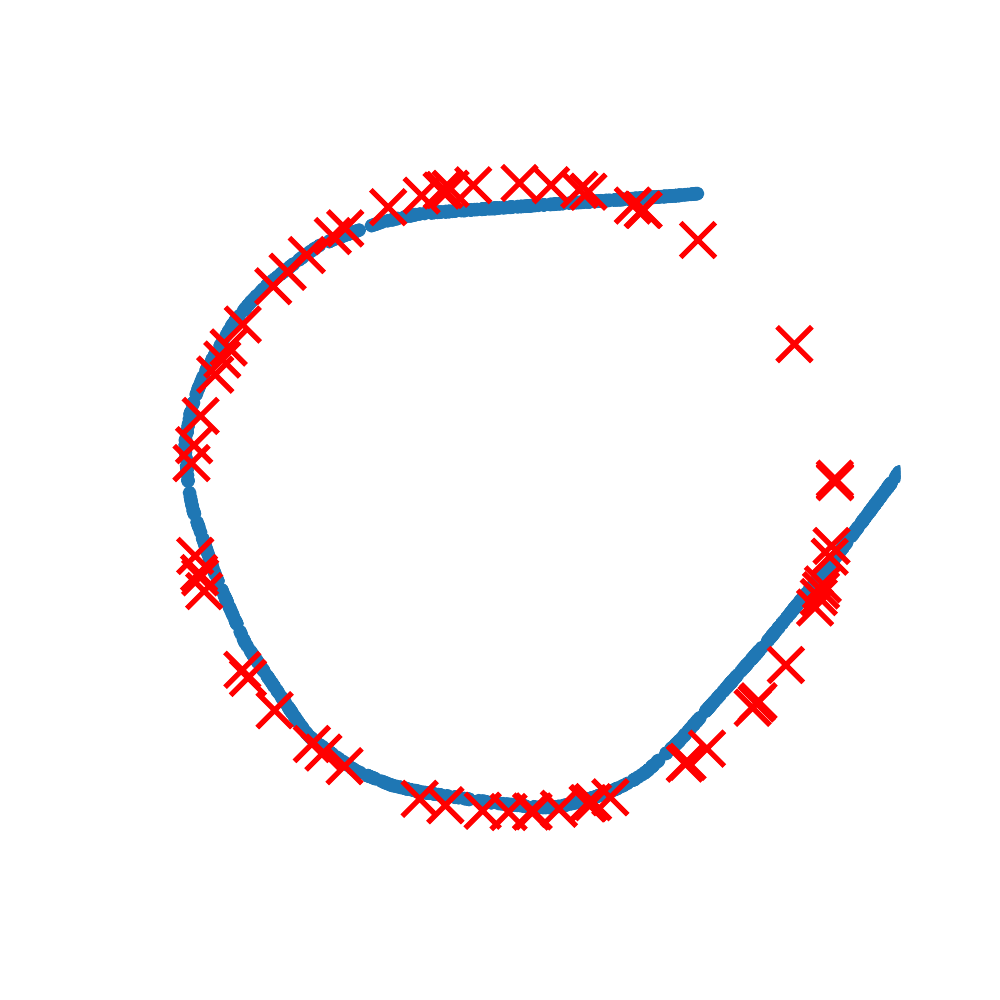}};
	\draw[->,ultra thick] (f1.east) -- (f2.west)
	node[node style] {FIT};
	\draw[->,dashed, ultra thick] (f2.south west) -- (f3.north east)
	node[node style] {OTS};
	\draw[->,ultra thick] (f3.east) -- (f4.west)
	node[node style,] {FIT};
	\draw[->,dashed, ultra thick] (f4.north east) -- (f5.south west)
	node[node style,] {OTS};
	\draw[->,ultra thick] (f5.east) -- (f6.west)
	node[node style,] {FIT};
	\draw[->,dashed, ultra thick] (f6.south west) -- (f7.north east)
	node[node style,] {OTS};
	\draw[->,ultra thick] (f7.east) -- (f8.west)
	node[node style,] {FIT};
	\end{tikzpicture}
          \ifarxiv
            \end{adjustbox}
          \else
          \fi
	\caption{The goal in this example is to fit the initial distribution (the blue
		central line) to the target distribution (the red outer ring).  The
		algorithm alternates OTS and FIT steps, first (OTS) associating input
		distribution samples with target distribution samples, and secondly (FIT)
		shifting input samples towards their targets, thereafter repeating the
		process.  Thanks to being gradual, and not merely sticking to the first or
		second OTS, the process has a hope of constructing a simple generator which
		generalizes well.}
	\label{fig:gradual:2}
\end{figure}
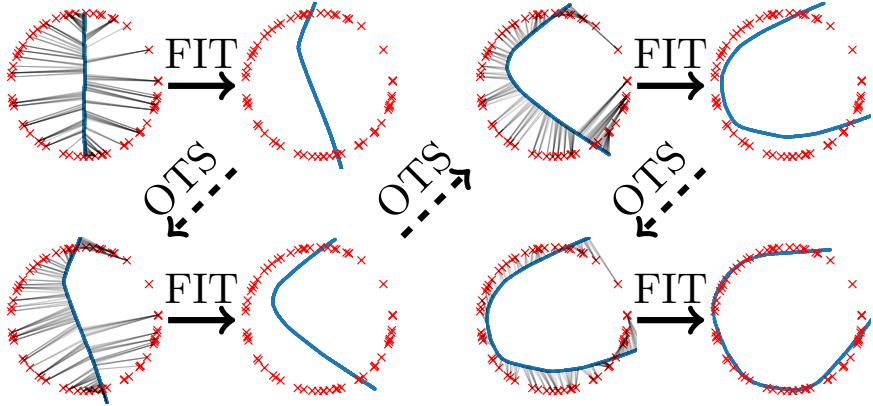

This paper proposes a simple non-adversarial but still alternating
procedure to fit generative networks
to target distributions.
The procedure
explicitly optimizes the Wasserstein-$p$ distance between the generator $g\#\mu$
and the target distribution $\hnu$.
As depicted in \Cref{fig:gradual:2},
it
alternates two steps:
given a current generator $g_i$,
an \emph{Optimal Transport Solver (OTS)} associates (or ``labels'') $g_i$'s probability mass
with that of the target distribution $\hnu$,
and then \emph{FIT} uses this labeling to find a new generator $g_{i+1}$
via a standard regression.

The effectiveness of this procedure hinges upon two key properties: it is
\textit{semi-discrete}, meaning the generators always give rise to continuous
distributions, and it is \textit{gradual}, meaning the generator is slowly
shifted towards the target distribution.  The key consequence of being
semi-discrete is that the underlying optimal transport can be realized with a
deterministic mapping.  Solvers exist for this problem and construct a
transport between the continuous distribution $g\#\mu$ and the target $\hnu$;
by contrast, methods forming only batch-to-batch transports using samples from $g\#\mu$
are biased and do
not exactly minimize the Wasserstein distance \citep{cramergan,sinkhorn_autodiff,OTGAN,
  WGANTS}.

The procedure also aims to be gradual, as in \Cref{fig:gradual:2}, slowly
deforming the source distribution into the target distribution.  While it is not
explicitly shown that this gradual property guarantees a simple generator,
promising empirical results measuring Wasserstein distance \emph{to a test set}
suggest that the learned generators generalize well.

\Cref{sec:algorithm} and \Cref{sec:analysis} detail the method along with a
variety of theoretical guarantees.  Foremost amongst these are showing that the
Wasserstein distance is indeed minimized, and secondly that it is minimized
with respect to not just the dataset distribution $\hnu$, but moreover the
underlying $\nu$ from which $\hnu$ was sampled.  This latter
property can be proved via the triangle inequality for Wasserstein distances,
however such an approach introduces the Wasserstein distance between $\nu$ and $\hnu$,
namely
$W(\nu,\hnu)$, which is exponential in dimension even in simple cases
\citep{SriperumbudurFGSL2012, generalization_equilibrium}.  Instead, we
show that when a parametric model captures the distributions well, then bounds
which are polynomial in dimension are possible.

Empirical results are presented in \Cref{sec:exp}. We find that our method
generates both quantitatively and qualitatively better digits than the compared
baselines on MNIST, and the performance is consistent on both training
and test datasets. We also experiment with the Thin-8 dataset \citep{thin8},
which is considered challenging for methods without a parametric loss.  We
discuss limitations of our method and conclude with some future directions.

 \section{Algorithm}
\label{sec:algorithm}

We present our alternating procedure in this section. We first describe the OTS and FIT steps in detail, and then give the overall algorithm.

\subsection{Optimal Transport Solver (OTS)}
\label{subsect:ots}

The Wasserstein-$p$ distance $W_p$ between two probability measures $\mu'$, $\nu'$ in a metric space $(X, d)$ is defined as
$$
W_p(\mu', \nu') := \left(\inf_{\gamma\in\Gamma(\mu',\nu')}\int_X d(x,y)^p\dif\gamma(x,y) \right)^{1/p},
$$
where $\Gamma(\mu', \nu')$ is the collection of probability measures on $X\times X$ with marginal distributions $\mu'$ and $\nu'$.
$W_p(\mu', \nu')$ is equal to the $\nicefrac 1 p$-th power of the \emph{optimal transport cost}
$$
\mathcal{T}_c(\mu',\nu'):=\inf_{\gamma\in\Gamma(\mu', \nu')}\int_X c(x, y)\dif\gamma(x,y)
$$
with cost function $c(x, y) := d(x, y)^p$.

By \emph{Kantorovich duality} \citep[Chapter 1]{villani_1},
\begin{equation}
\mathcal{T}_c(\mu',\nu') =\sup_{\varphi, \psi\in\Phi_c} \int_X \varphi(x)\mathrm{d}\mu'(x) + \int_X \psi(y)\mathrm{d}\nu'(y),
\label{eq:kantorovich_duality}
\end{equation}
where $\Phi_c$ is the collection of $(\varphi, \psi)$ where $\varphi\in L_1(\mathrm{d}\mu')$ and $\psi\in L_1(\mathrm{d}\nu')$ (which means $\varphi$ and $\psi$ are absolutely Lebesgue integrable functions with respect to $\mu', \nu'$), and $\varphi(x) + \psi(y) \leq c(x, y)$ for almost all $x, y$).

In our generative modeling case, $\mu'=g\#\mu$ is a pushforward of the simple distribution $\mu$ by $g$, and $\nu'=\hnu$
is an \emph{empirical measure}, meaning the uniform distribution on a training set $\{y_1,...,y_N\}$.
When $g\#\mu$ is continuous,
the optimal transport
problem here becomes \emph{semi-discrete}, and the maximizing choice of
$\varphi$ can be solved analytically, transforming the problem to optimization
over a vector in $\R^N$ \citep{cuturi_stochastic_ot, cuturi_book}:
\begin{align}
&\nonumber \mathcal{T}_c(g\#\mu,\hnu)
=\sup_{\varphi, \psi\in\Phi_c} \int_X\varphi(x)\mathrm{d}g\#\mu(x) + \frac{1}{N}\sum_{i=1}^N\psi(y_i)\\
\nonumber&=\sup_{\varphi\in\Phi'_{c,\hpsi}, \hpsi\in\R^N} \int_X \varphi(x)\mathrm{d}g\#\mu(x) + \frac{1}{N}\sum_{i=1}^N\hpsi_i\\
&=\sup_{\hpsi\in\R^N}\int_X\min_i(c(x, y_i)-\hpsi_i)\mathrm{d}g\#\mu(x) + \frac{1}{N}\sum_{i=1}^N\hpsi_i,
\label{eq:kantorovich_duality_semi_discrete}
\end{align}
where $\hpsi_i:=\psi(y_i)$, and $\Phi'_{c,\hpsi}$ is the collection of functions $\varphi\in L_1(\mathrm{d}(g\#\mu))$ such that $\varphi(x)+\hpsi_i \leq c(x, y_i)$ for almost all $x$ and $i=1,...,N$. The third equality comes from the maximizing choice of $\varphi$: $\varphi(x)=\min_i(c(x, y_i)-\hpsi_i):=\psi^c(x)$, the \emph{c-transform} of $\psi$.

Our OTS solver, presented in \Cref{alg:ots}, uses SGD to maximize \cref{eq:kantorovich_duality_semi_discrete},
or rather to minimize its negation
\begin{equation}
-\int_X\min_i(c(x, y_i)-\hpsi_i)\mathrm{d}g\#\mu(x) - \frac{1}{N}\sum_{i=1}^N\hpsi_i.
\label{eq:ots_loss}
\end{equation}
OTS is similar to Algorithm 2 of \cite{cuturi_stochastic_ot}, but without averaging.
Note as follows that \Cref{alg:ots} is convex in $\hpsi$, and thus a convergence theory
of OTS could be developed, although this direction is not pursued here.

\begin{algorithm}[h]
	\caption{Optimal Transport Solver (OTS)}
	\label{alg:ots}
	\begin{algorithmic}
		\STATE {\bfseries Input:} continuous generated distribution $g\#\mu$, training dataset $(y_1,...,y_n)$ corresponding to $\hnu$, cost function $c$, batch size $B$, learning rate $\eta$.
		\STATE {\bfseries Output:} $\hpsi=(\hpsi_1,...,\hpsi_N)$
		\STATE Initialize $t := 0 $ and $\hpsi^{(0)} \in \R^N$.
		\REPEAT
			\STATE Generate samples $\mathbf{x}=(x_1,...,x_B)$ from $g\#\mu$.
			\STATE Define loss $l(\mathbf{x}) := \frac{1}{B}\sum_{j=1}^B\min_i(c(x_j, y_i) - \hpsi_i) + \frac{1}{N}\sum_{i=1}^N\hpsi_i$.
			\STATE Update $\hpsi^{(t+1)} := \hpsi^{(t)} + \eta\cdot\nabla_{\hpsi}l(\mathbf{x})$.
			\STATE Update $t := t + 1$.
		\UNTIL Stopping criterion is satisfied.
		\RETURN $\hpsi^{(t)}$.
	\end{algorithmic}
\end{algorithm}
 
 \begin{proposition}
 	\Cref{eq:ots_loss} is a convex function of $\hpsi$.
 \end{proposition}
 \begin{proof}
   It suffices to note that
$\min_i \del[1]{ c(g(x),y_i) - \hpsi_i }$
   is a minimum of concave functions and thus concave;
   \cref{eq:ots_loss} is therefore concave since it is a convex combination of concave functions with an
   additional linear term.
\end{proof}
In the \emph{semi-discrete} setting where $g\#\mu$ is continuous and $\hnu$ is discrete, it can be proved that the \emph{Kantorovich optimal transference plan} computed via \cref{eq:kantorovich_duality} is indeed a \emph{Monge optimal transference plan} characterized by $\argmin_i (x, y_i)-\hpsi_i$, which is a deterministic mapping providing the regression target for our FIT step.

\begin{proposition}
	\label{prop:argmin}
	Assume $X=\R^d$, $g\#\mu$ is continuous, and the cost function $c(x,y)$ takes the form of $c(x-y)$ and is a strictly convex, superlinear function on $\R^d$. Given the optimal $\hpsi$ for \cref{eq:ots_loss}, then  $T(x):=y_{\argmin_i c(x,y_i)-\hpsi_i}$, which is a Monge transference plan, is the unique Kantorovich optimal transference plan from $g\#\mu$ to $\nu$.
\end{proposition}
(The proof is technical, and appears in the Appendix.)

We give some remarks to \Cref{prop:argmin}. First, Wasserstein-$p$ distances on $\ell_p$ metric with $p > 1$ satisfy strict convexity and superlinearity \citep{gangbo1996geometry}, while $p=1$ does not. On the other hand, in practice we have found that for $p=1$, \Cref{alg:ots} still converges to near-optimal transference plans, and this particular choice of metric generates crisper images than others.

Second, the continuity of $g\#\mu$, which is required for the uniqueness of OT plans may be violated in practice, but
can be theoretically circumvented by adding an arbitrarily small perturbation to $g$'s output. On the other hand, since the optimal plan lies in the set $\varphi(x)+\psi(y)=c(x,y)$ almost surely \citep[Theorem 5.10]{villani_2}, it has to have the ``argmin form'' in \Cref{prop:argmin}. Thus, the only condition in which $T(x):=y_{\argmin_i c(x,y_i)-\hpsi_i}$ does not characterize a unique Monge OT plan, is the existence of ties when computing $\argmin$, which does not happen in practice. We give some additional discussion about this issue in the Appendix. 

A drawback of \Cref{alg:ots} is that computing minimum on the whole dataset has
$O(N)$ complexity, which is costly for extremely large datasets. We will revisit
this issue in \Cref{sec:exp}.

\subsection{Fitting the Optimal Transference Plan (FIT)}
\label{subsect:fit}

Given an initial generator $g$, and an optimal transference plan $T$ between
$g\#\mu$ and $\hnu$ thanks to OTS, we update $g$ to obtain a
new generator $g'$ by simply sampling $z \sim \mu$ and regressing the \emph{new}
generated sample $g'(z)$ towards the \emph{old} OT plan $T(g(z))$, as detailed
in \Cref{alg:fit}.

Under a few assumptions detailed in \Cref{sec:optimization},
\Cref{alg:fit} is guaranteed to return a generator $g'$ with strictly lesser optimal
transport cost
$\cT_c(g'\#\mu, \hnu) \leq \mathbb{E}_{x\sim g'\#\mu} c(x, T(x)) < \mathbb{E}_{x\sim g\#\mu} c(x, T(x)) = \cT_c(g\#\mu, \hnu),$
where $T$ denotes an exact optimal plan between $g\#\mu$ and $\hnu$; \Cref{sec:optimization} moreover considers the case of a merely approximately optimal $T$, as returned by OTS.

\begin{algorithm}[h]
  \caption{Fitting Optimal Transport Plan (FIT)}
	\label{alg:fit}
	\begin{algorithmic}
		\STATE {\bfseries Input:} sampling distribution $\mu$, old generator $g$ with parameter $\theta$, transference plan $T$, cost function $c$, batch size $B$, learning rate $\eta$.
		\STATE {\bfseries Output:} new generator $g'$ with parameter $\theta'$.
		\STATE Initialize $t := 0$ and $g'$ with parameter $\theta'^{(0)} = \theta$.
		\REPEAT
		\STATE Generate random noise $\mathbf{z}=(z_1,...,z_B)$ from $\mu$.
		\STATE Define loss $l(\mathbf{z}) := \frac{1}{B}\sum_{j=1}^B c(g'(z), T(g(z)))$.
		\STATE Update $\theta^{(t+1)} := \theta'^{(t)} - \eta\cdot\nabla_{\theta'}l(\mathbf{z})$.
		\STATE Update $t := t + 1$.
		\UNTIL Stopping criterion is satisfied.
		\RETURN $g'$ with parameter $\theta'^{(t)}$.
	\end{algorithmic}
\end{algorithm}

\subsection{The Overall Algorithm}
\label{sect:overall}

The overall algorithm, presented in \Cref{alg:overall},
alternates between OTS and FIT: during
iteration $i$, OTS solves for the optimal transport map $T$ between old
generated distribution $g_{i}\#\mu$ and $\hnu$, then FIT regresses $g_{i+1}\#\mu$
towards $T\#g_{i}\#\mu$ to obtain lower Wasserstein distance.

\begin{algorithm}[h]
	\caption{Overall Algorithm}
	\label{alg:overall}
	\begin{algorithmic}
		\STATE {\bfseries Input:} sampling distribution $\mu$, training dataset $(y_1,...,y_n)$ corresponding to $\hnu$, initialized generator $g_0$ with parameter $\theta_0$, cost function $c$, batch size $B$, learning rate $\eta$.
		\STATE {\bfseries Output:} final generator $g$ with parameter $\theta$.
		\STATE Initialize $i := 0$ and $g_0$ with parameter $\theta^{(0)}=\theta_0$.
		\REPEAT
			\STATE Compute $\hpsi_i := \text{OTS}(g_{i}\#\mu, \hnu, c, B, \eta)$.
			\STATE Get $T_i$ as $T_i(x) := \argmin_{y_i}c(x,y_i)-\hpsi_i$.
			\STATE Compute $g_{i+1} := \text{FIT}(\mu, g_{i}, T_i, c, B, \eta)$ with parameter $\theta^{(i+1)}$.
			
			\STATE Update $i := i + 1$.
		\UNTIL Stopping criterion is satisfied.
		\RETURN $g$ with parameter $\theta^{(i)}$.
	\end{algorithmic}
\end{algorithm}

 \section{Theoretical Analysis}
\label{sec:analysis}

We now analyze the optimization and generalization properties of our \Cref{alg:overall}: we will show that the method indeed minimizes the empirical transport
cost, meaning $\cT_c(g_i\#\mu,\hnu)\to 0$, and also generalizes to the transport cost over the underlying distribution, meaning $\cT_c(g_i\#\mu,\nu)\to 0$.

\subsection{Optimization Guarantee: $\cT_c(g_i\#\mu,\hnu)\rightarrow 0$.}
\label{sec:optimization}
Our analysis works for costs $\cT_c$ whose $\beta$-th powers satisfy the triangle inequality, such as $\cT_c:=W_p^p$ over any metric space and $\beta=\nicefrac 1 p$, if $p\geq 1$ \citep[Theorem 7.3]{villani_1}.

\def\epsota{\eps_{\textup{ot1}}}
\def\epsotb{\eps_{\textup{ot2}}}
\def\epsotb{\eps_{\textup{ot2}}}
\def\epsfit{\eps_{\textup{fit}}}

Our method is parameterized by a scalar $\alpha \in(0,\nicefrac 1 2)$ whose role is to
determine the relative precisions of OTS and FIT, controlling the gradual property of
our method.  Defining $C_\mu(f,g) := \int c(f(x), g(x))\dif\mu(x)$, we assume that for each round $i$, there exist error terms $\epsota, \epsotb, \epsfit$ such that:

\begin{enumerate}[topsep=0pt, label={(\arabic*)}, wide]
	\item Round $i$ of OTS finds transport $T_i$ satisfying
	$$ \cT_c^\beta(g_{i}\#\mu, \hnu) \leq C_\mu^\beta(T_i\circ g_i, g_i) \leq \cT_c^\beta
	(g_{i}\#\mu, \hnu)(1 + \epsota)$$ (approximate optimality), and $\cT_c^\beta(T_i\#g_{i}\#\mu,\hnu) \leq \epsotb \leq \alpha \cT_c^\beta(g_i\#\mu,\hnu)$ (approximate pushforward);

	\item Round $i$ of FIT finds $g_{i+1}$ satisfying
	\begin{align*}
	&C_\mu^\beta(T_i\circ g_{i}, g_{i+1}) \leq \epsfit\leq \frac{1-2\alpha}{1+\epsota}C_\mu^\beta(T_i\circ g_{i}, g_{i})\\
	&\leq (1-2\alpha)\cT_c^\beta(g_{i}\#\mu, \hnu)
	\text{ (progress of FIT).}
	\end{align*}
\end{enumerate}
\vspace*{-7pt}
In addition, we assume each $g_i\#\mu$ is continuous to guarantee the existence of Monge transport plan.

(1) is satisfied by \Cref{alg:ots} since it represents a convex problem;
moreover, it is necessary in practice to assume only approximate solutions.
(2) holds when there is still room for the generative network to improve Wasserstein distance: otherwise, the training process can be stopped.

$\alpha$ is a tunable parameter of our overall algorithm: a large $\alpha$ relaxes the optimality requirement of OTS (which allows early stopping of \Cref{alg:ots}) but requires large progress of FIT (which prevents early stopping of \Cref{alg:fit}), and vice versa. This
gives us a principled way to determine the stopping criteria of OTS and
FIT.

Given the assumptions, we now show $\cT_c(g_i\#\mu,\hnu)\rightarrow 0$. By triangle inequality,
\begin{align*}
&\cT_c^\beta(g_{i+1}\#\mu, \hnu)\\
\leq\ &
\cT_c^\beta(g_{i+1}\#\mu, T_i\#g_i\#\mu)
+ \cT_c^\beta(T_i\#g_i\#\mu, \hnu)\\
\leq\ &
\cT_c^\beta(g_{i+1}\#\mu, T_i\#g_i\#\mu)
+ \epsotb.
\end{align*}

Since $g_{i+1}\#\mu$ is continuous, $\cT_c(g_{i+1}\#\mu, T_i\#g_{i}\#\mu)$
is realized by some deterministic transport $T_i'$ satisfying
$T_i'\#g_{i+1}\#\mu = T_i\#g_i\#\mu$, whereby
\begin{align*}
\cT_c^\beta(g_{i+1}\#\mu, T_i\#g_i\#\mu)
&=
C_\mu^\beta(T_i'\circ g_{i+1}, g_{i+1})
\\
&=
C_\mu^\beta(T_i\circ g_{i}, g_{i+1})\leq
\epsfit.
\end{align*}
Combining these steps with the upper bounds on $\epsotb$ and $\epsfit$,
\begin{align*}
\cT_c^\beta(g_{i+1}\#\mu, \hnu)
\leq \epsotb + \epsfit
&\leq (1-\alpha) \cT_c^\beta(g_i\#\mu, \hnu)\\
&\leq e^{-\alpha} \cT_c^\beta(g_i\#\mu, \hnu).
\end{align*}
Summarizing these steps and iterating this inequality gives the following bound
on $\cT_c(g_t\#\mu,\hnu)$, which goes to $0$ as $t\to0$.

\begin{theorem}
	\label{fact:iter}
	Suppose (as discussed above) that $\cT_c^\beta$ satisfies the triangle inequality,
	each $g_i\#\mu$ is continuous, and
	the OTS and FIT iterations satisfy (1) (2),
then $\cT_c(g_t\#\mu, \hnu) \leq e^{\nicefrac{-t\alpha}{\beta}} \cT_c(g_0\#\mu,\hnu)$.
\end{theorem}

\subsection{Generalization Guarantee: $\cT_c(g_i\#\mu,\nu)\to 0$.}
\label{sect:generalization}

\newcommand\BOUND{\ensuremath{D_{n,\Psi}}}

In the context of generative modeling, \emph{generalization} means that the model fitted via training dataset $\hnu$ not only has low divergence $\cD(g_i\#\mu, \hnu)$ to $\hnu$, but also low divergence to $\nu$, the underlying distribution from which $\hnu$ is drawn \emph{i.i.d.}. If $\cT_c$ satisfies triangle inequality, then
\[
\cT_c(g_i\#\mu,\nu)
\leq \cT_c(g_i\#\mu,\hnu)
+ \cT_c(\hnu, \nu),
\]
and the second term goes to 0 with sample size $n\rightarrow\infty$, but the sample complexity depends exponentially on the dimensionality \citep{SriperumbudurFGSL2012, generalization_equilibrium, thin8}.  To remove this exponential dependence, we make parametric assumptions about the underlying distribution $\nu$; a related idea was investigated in detail
in parallel work \citep{bai2018approximability}.

Our approach is to assume the \emph{Kantorovich potential} $\hpsi$, defined on $\hnu$, is induced from a function $\psi\in\Psi$ defined on $\nu$, where $\Psi$ is a function class with certain approximation and generalization guarantees. Since neural networks are one such function classes (as will be discussed later), this is an empirically verifiable assumption (by fitting a neural network to approximate $\hpsi$), and is indeed verified in the Appendix.

For this part we use slightly different notation: for a fixed sample size $n$, let $(g_n, T_n, \hnu_n)$ denote the earlier $(g,T,\hnu)$. We first suppose the following \emph{approximation condition}:
Suppose that for any $\eps >0$, there exists a class of functions $\Psi$ so that
\begin{equation}
\sup_{\psi \in L_1(\nu)}
\int \psi^c \dif\mu + \int \psi\dif\nu
\leq
\eps
+
\sup_{\psi \in \Psi}
\int \psi^c \dif\mu + \int \psi\dif\nu;
\label{eq:cond:apx}
\end{equation}
thanks to the extensive literature on function approximation with neural networks \citep{nn_stone_weierstrass,cybenko,yarotsky}, there are various ways to guarantee this, for example increasing the depth of the network. A second assumption is a \emph{generalization condition}:
given any sample size $n$ and function class $\Psi$,
suppose there exists $\BOUND\geq 0$ so that
with probability at least $1-\delta$ over a draw of $n$ examples from $\nu$
(giving rise to empirical measure $\hnu_n$),
every $\psi\in\Psi$ satisfies
\begin{equation}
\int \psi \dif\nu \leq \BOUND + \int \psi \dif\hnu_n;
\label{eq:cond:gen}
\end{equation}
thanks to the extensive theory of neural network generalization,
there are in turn various ways to provide such a guarantee \citep{anthony_bartlett_nn}, for example through VC-dimension of neural networks.

Combining these two assumptions,
\begin{align*}
&\cT_c(g_n\#\mu, \nu)=
\sup_{\psi\in L_1(\nu)}\cbr{
	\int \psi^c \dif(g_n\#\mu) + \int \psi \dif\nu
}
\\
&\leq
\eps
+
\sup_{\psi\in \Psi}\cbr{
	\int \psi^c \dif(g_n\#\mu) + \int \psi \dif\nu
}
\\
&\leq
\BOUND
+
\eps
+
\sup_{\psi\in \Psi}\cbr{
	\int \psi^c \dif(g_n\#\mu) + \int \psi \dif\hnu_n
}
\\
&\leq
\BOUND
+
\eps
+
\sup_{\psi\in L_1(\hnu_n)}\cbr{
	\int \psi^c \dif(g_n\#\mu) + \int \psi \dif\hnu_n
}
\\
&\leq
\BOUND
+
\eps
+
\cT_c(g_n\#\mu, \hnu_n).
\end{align*}
This can be summarized as follows.
\begin{theorem}
	Let $\eps > 0$ be given,
	and suppose assumptions \cref{eq:cond:apx,eq:cond:gen} hold.
Then, with probability at least $1-\delta$ over the draw of $n$ examples from $\nu$,
	\[
	\cT_c(g_n\#\mu, \nu)
	\leq
	\BOUND
	+
	\eps
	+
	\cT_c(g_n\#\mu, \hnu_n).
	\]
\end{theorem}

By the earlier discussion, $\BOUND\to 0$ and $\eps\to 0$ as $n\to\infty$, whereas
the third term goes to 0 as discussed in \Cref{sec:optimization}.

 \section{Experimental Results}
\label{sec:exp}

\subsection{Experimental Setup}

We briefly describe the datasets, baselines, and metrics in our
experiments.  Detailed descriptions of network architectures and the
computations of evaluation metrics can be found in the Appendix.

\paragraph{Datasets.} We evaluate our generative model on the MNIST
and Thin-8 $128\times128$ datasets \citep{mnist,thin8}.
On MNIST, we use the original test/train split \citep{mnist},
and each model is
trained on the training set and evaluated on both training and testing sets.
For Thin-8 we use the full dataset for training since the number of samples is
limited. 

\paragraph{Baselines.}
We compare our model against several neural net generative models:
(1) WGAN \citep{WGAN}; 
(2) WGANGP \citep{WGANGP};
(3) variational autoencoder (VAE) \citep{VAE};
(4) Wasserstein autoencoder (WAE) \citep{WAE}.
We experiment with both MLP and DCGAN as the generator
architecture \citep{DCGAN}, and use DCGAN as the default discriminator/encoder architecture
as it achieves better results for these baselines. 
Our method and WAE allow optimizing general optimal transport costs, and we
choose to optimize the Wasserstein-1 distance on the $\ell_1$ metric both for fair
comparison with WGAN, and also since we observed clearer images on both MNIST and Thin-8.

\paragraph{Metrics.} We use the following metrics to quantify the
performance of different generative models: (1) Neural net distance (NND-WC,
NND-GP) \citep{generalization_equilibrium} based on DCGAN with weight clipping
and gradient penalty respectively; (2) Wasserstein-1 distance (WD) on $\ell_1$
metric; (3) Inception score (IS) \citep{improvegan}; and (4) Fr\'echet Inception
distance (FID) \citep{FID}. 

We chose the above metrics because they capture different aspects of a
generative model, and none of them is a one-size-fit-all evaluation measure.
Among them, NND-WC and NND-GP are based on the adversarial game and thus biased
in favor of WGAN and WGANGP. WD measures the Wasserstein distance between the
generated distribution and the real dataset, and favors WAE and our
method. IS and FID can be considered as neutral evaluation metrics, but they
require labeled data or pretrained models to measure the performance of
different models.

\subsection{Qualitative Study}

We first qualitatively investigate our generative model and compare the samples
generated by different models.

\begin{figure*}[t]
	\centering {\
		\subfigure[Real Image] {
			\label{fig:mnist-a}
			\includegraphics[width=0.14\columnwidth]{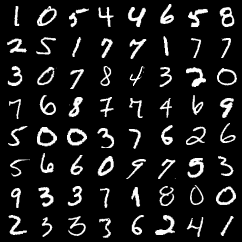}
		}
		\subfigure[Ours-MLP] {
			\label{fig:mnist-b}
			\includegraphics[width=0.14\columnwidth]{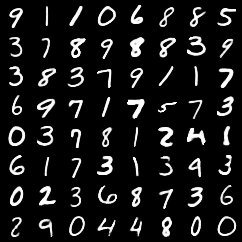}
		}
		\subfigure[Ours-DCGAN] {
			\label{fig:mnist-c}
			\includegraphics[width=0.14\columnwidth]{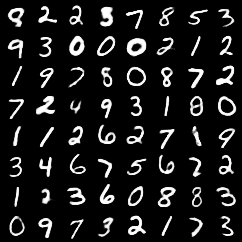}
		}
		\subfigure[WGAN] {
			\label{fig:mnist-d}
			\includegraphics[width=0.14\columnwidth]{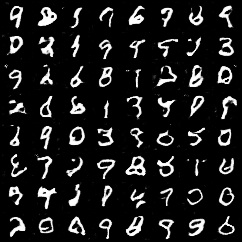}
		}
		\subfigure[WGANGP] {
			\label{fig:mnist-e}
			\includegraphics[width=0.14\columnwidth]{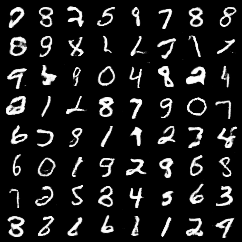}
		}
		\subfigure[WAE] {
			\label{fig:mnist-f}
			\includegraphics[width=0.14\columnwidth]{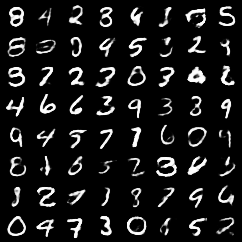}
		}
	}
	\caption{
		Real and generated samples on the MNIST dataset: (a) real samples;
		(b) samples generated by our model with MLP as the generator network; (c)
		samples generated by our model with DCGAN as the generator network; (d)
		samples generated by WGAN; (e) samples generated by WGANGP; (f) samples
		generated by WAE. DCGAN is used as the generator architecture in (d)(e)(f).
	}
	\label{fig:mnist_qualitative}
\end{figure*}

\begin{figure*}[t]
	\centering {\
		\subfigure[\small Ours-MLP] {
			\label{fig:thin8-b}
			\includegraphics[width=0.14\columnwidth]{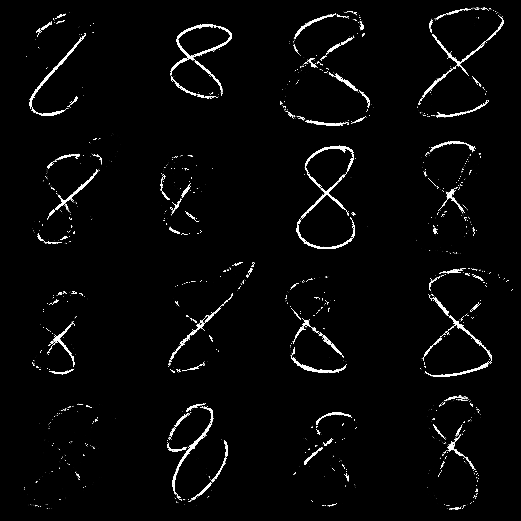}
		}
		\subfigure[\small WGANGP-MLP] {
			\label{fig:thin8-c}
			\includegraphics[width=0.14\columnwidth]{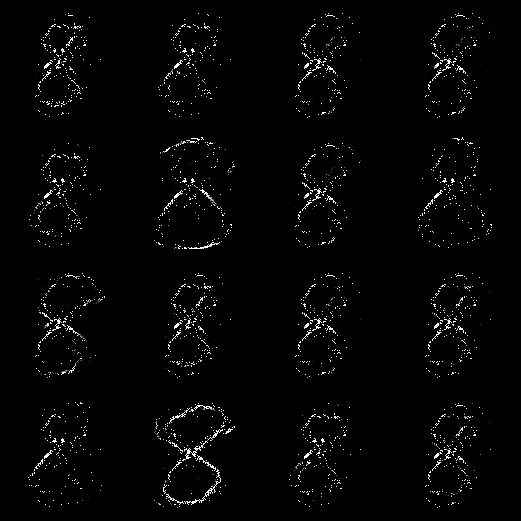}
		}
		\subfigure[\small WAE-MLP] {
			\label{fig:thin8-d}
			\includegraphics[width=0.14\columnwidth]{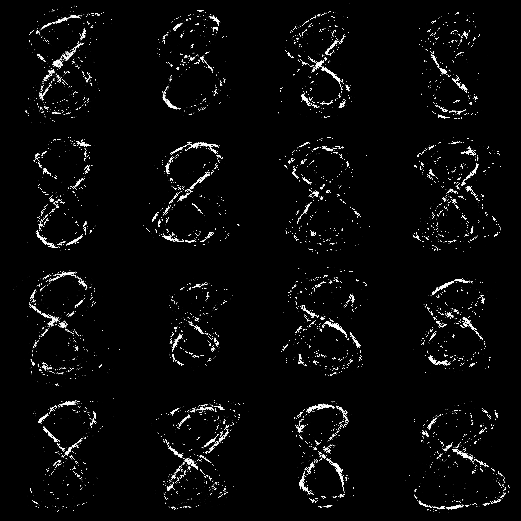}
		}
		\subfigure[\small Ours-DCGAN] {
			\label{fig:thin8-e}
			\includegraphics[width=0.14\columnwidth]{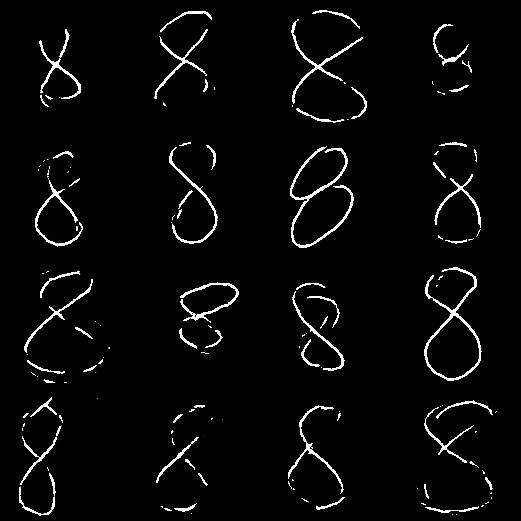}
		}
		\subfigure[\small WGANGP-DCGAN] {
			\label{fig:mnist-f}
			\includegraphics[width=0.14\columnwidth]{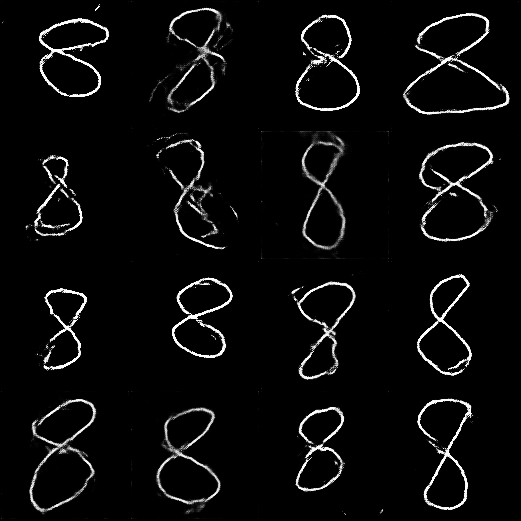}
		}
		\subfigure[\small WAE-DCGAN] {
			\label{fig:mnist-f}
			\includegraphics[width=0.14\columnwidth]{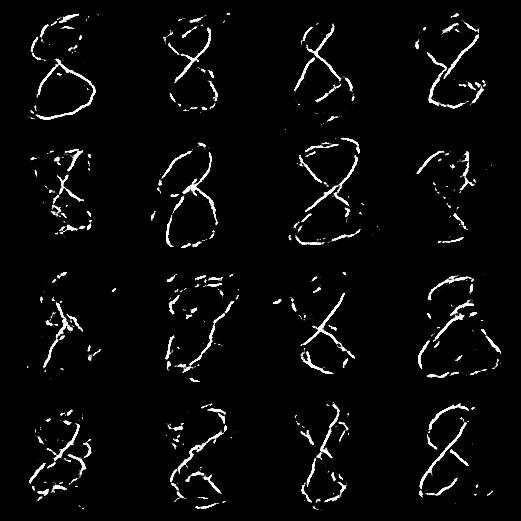}
		}
	}
	\caption{
		Generated samples on the Thin-8 dataset. (a)(b)(c) are samples
		generated by different methods using MLP as the generator; and
		(d)(e)(f) are samples  when using DCGAN as the generator.
	}
	\label{fig:thin8_qualitative}
\end{figure*}

\paragraph{Samples of generated images.} \Cref{fig:mnist_qualitative}
shows samples generated by different models on the MNIST dataset.  The results
show that our method with MLP (\Cref{fig:mnist-b}) and DCGAN
(\Cref{fig:mnist-c}) both generate digits with better visual quality than the
baselines with the DCGAN architecture.  \Cref{fig:thin8_qualitative} shows the
generated samples on Thin-8 by our method, WGANGP, and WAE. The results of
WGAN and VAE are omitted as they are similar to both WGANGP and WAE
consistently on Thin-8.  When MLP is used as the generator architecture, our
method again outperforms WGANGP and WAE in terms of the visual quality of the
generated samples. When DCGAN is used, the digits generated by our method have
slightly worse quality than WGANGP, but better than WAE.

\paragraph{Importance of alternating procedure.} We use this set of
experiments to verify the importance of the alternating procedure that
gradually improves the generative network.  \Cref{fig:compare} shows: (a) the
samples generated by our model; and (b) the samples generated by a weakened
version of our model that does not employ the alternating procedure.  The
non-alternating counterpart derives an optimal transport plan in the first run,
and then fits towards the derived plan. It can be seen clearly the samples
generated with such a non-alternating procedure have considerably lower visual
quality. This verifies the importance of the alternating training procedure: fitting the generator
towards the initial OT plan does not provide good enough gradient direction to produce a
high-quality generator.

\begin{figure}
  \centering {\
    \subfigure[Alternating.] {
      \label{fig:a}
      \includegraphics[width=0.3\columnwidth]{mnist_our_mlp.png}
    }
	\hfil
    \subfigure[Non-alternating.] {
      \label{fig:b}
      \includegraphics[width=0.3\columnwidth]{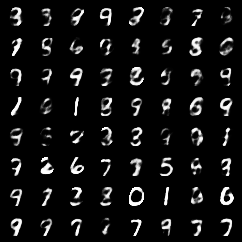}
    }
  }
  \caption{Generated samples on MNIST with and without the alternating procedure.}
  \label{fig:compare}
\end{figure}

\subsection{Quantitative Results}

We proceed to measure the quantitative performance of the compared
models.

\paragraph{MNIST results.} \Cref{tbl:mnist_training} shows the
performance of different models on the MNIST training and testing sets. In the
first part when MLP is used to instantiate generators, our model achieves the
best performance in terms of all the five metrics. The results on neural
network distances (NND-WC and NND-GP) are particularly interesting: even though
neural network distances are biased in favor of GAN-based models because the
adversarial game explicitly optimizes such distances, our model still
outperforms GAN-based models without adversarial training. The second part
shows the results when DCGAN is the generator architecture.  Under this
setting, our method achieves the best results among all the metrics except for
neural network distances.  Comparing the performance of our method on the
training and testing sets, one can observe its consistent performance and
similar comparisons against baselines.  This phenomenon empirically verifies
that our method does not overfit.

\begin{table*}[t]
\caption{Quantitative results on the MNIST training and testing sets. Note
  that the Wasserstein distance on training and testing sets of different sizes
  are not directly comparable.
}
	\begin{center}
          \ifarxiv
            \begin{adjustbox}{width=0.98\textwidth}
          \else
            \begin{small}
          \fi
	\begin{sc}
	\begin{tabular}{c | c | c c c c c | c c c c }
		\toprule
                \textbf{Method} & \textbf{Arch} & \multicolumn{5}{c}{ \textbf{MNIST Training} } & \multicolumn{4}{c}{\textbf{MNIST Test}} \\ \midrule
		 &  &  \textbf{NND-WC} & \textbf{NND-GP} & \textbf{WD} & \textbf{IS} & \textbf{FID} &  \textbf{NND-WC} & \textbf{NND-GP} & \textbf{WD} & \textbf{FID}\\
		\midrule
                WGAN &  & 0.29 & 5.82 & 140.710 & 7.51 & 31.28 & 0.29 & 6.05 & 142.48 & 31.91\\
                WGANGP & & 0.13 & 2.61 & 107.61 & 8.89 & 8.46 & 0.12 & 3.02 & 112.22 & 8.99\\
		VAE & MLP &0.53 & 4.26 & 101.06 & 7.10 & 52.42 &0.52 & 4.42 & 110.49 & 51.88\\
                WAE & &0.18 & 3.64 & 90.91 & 8.42& 11.12 &0.15 & 3.80 & 101.46 & 11.49\\
		Ours & &\textbf{0.11} & \textbf{2.56} & \textbf{66.68} & \textbf{9.77}& \textbf{3.21}&\textbf{0.10} & \textbf{2.79} & \textbf{82.87} &  \textbf{3.56}\\
		\midrule
		WGAN & & 0.11 & 4.69 & 125.63 & 7.02& 27.64& 0.10 & 4.86 & 132.97 & 28.44\\
                WGANGP & & \textbf{0.08} & \textbf{0.83} & 93.61 & 8.65 & 4.65& \textbf{0.07} & \textbf{1.66} & 104.15 & 5.45\\
		VAE & DCGAN & 0.48 & 3.68 & 106.63 & 6.96& 42.10& 0.46 & 3.89 & 115.59 & 41.95\\
		WAE & & 0.18 & 3.29 & 90.96 & 8.35& 12.28& 0.15 & 3.53 & 101.02 & 12.66\\
		Ours & & 0.10& 2.28& \textbf{70.13}& \textbf{9.54} & \textbf{3.76}& 0.09& 2.55& \textbf{82.79}& \textbf{4.18}\\
		\bottomrule
	\end{tabular}
	\end{sc}
        \ifarxiv
          \end{adjustbox}
        \else
          \end{small}
        \fi
\end{center}
\label{tbl:mnist_training}
\end{table*}

\paragraph{Thin-8 results.} There are no meaningful classifiers to compute IS and
FID on the Thin-8 dataset. We thus only use NND-WC, NND-GP and WD as the
quantitative metrics, and \Cref{tbl:thin8} shows the results. Our method
obtains the best results among all the metrics with both the MLP and DCGAN
architectures. For NND-WC, all methods expect ours have similar results of
around 3.1: we suspect this is due to the weight clipping effect, which is
verified by tuning the clipping factor in our exploration. NND-GP and WD have
consistent correlations for all the methods. This phenomenon is expected on a
small-sized but high-dimensional dataset like Thin-8, because the discriminator
neural network has enough capacity to approximate the Lipschitz-1 function class on
the samples.  The result comparison between NND-WC and NND-GP directly supports
the claim \cite{WGANGP} that gradient-penalized neural networks (NND-GP) has
much higher approximation power than weight-clipped neural networks (NND-WC).

It is interesting to see that WGAN and WGANGP lead to the largest neural net
distance and Wasserstein distance, yet their generated samples still have the
best visual qualities on Thin-8.  This suggests that the success of GAN-based
models cannot be solely explained by the restricted approximation power of
discriminator \citep{generalization_equilibrium, thin8}.

\begin{table}[h]
        \caption{Quantitative results on the Thin-8 dataset.}
	\begin{center}
		\begin{small}
			\begin{sc}
				\begin{tabular}{c | c | c c c}
					\toprule
					\textbf{Method} & \textbf{Arch} &  \textbf{NND-WC} & \textbf{NND-GP} & \textbf{WD} \\
					\midrule
					WGAN &  & 3.12 & 258.05 & 3934\\
					WGANGP & & 3.12 & 144.38 & 2235\\
					VAE & MLP &3.11 & 105.22 & 1950\\
					WAE & &3.07 & 111.79 & 1945 \\
                                        Ours & &\textbf{2.87} & \textbf{80.48} & \textbf{1016}\\
					\midrule
					WGAN & & 3.10 & 157.84 & 2481\\
					WGANGP & & 3.04 & 79.47 &1909 \\
					VAE & DCGAN & 3.02 & 81.38 & 1820\\
					WAE & & 3.11 & 88.04 & 1950\\
                                        Ours & &\textbf{2.92} & \textbf{47.59}& \textbf{923}\\
					\bottomrule
				\end{tabular}
			\end{sc}
		\end{small}
	\end{center}
	\label{tbl:thin8}
\end{table}

\paragraph{Time cost.} \Cref{tab:time} reports the training time of
different models on MNIST.  For moderate sized datasets such as MNIST, our
method is faster than WGAN and WGANGP but slower than VAE and WAE.  Compared
with GAN-based models, our method does not have a discriminator which saves time.  On the other
hand, the loss function of \cref{eq:ots_loss} requires computing
$c(x,y_i)-\hpsi_i$ for the whole dataset, which can be costly. That being said,
a useful trick to accelerate our model is to take gradient of $\hpsi$ over a moderately-sized
subsample (for example 1\%) of dataset in the first iterations, then gradually increase the
subsample data size to cover the whole dataset.

\begin{table}
\caption{Training time per iteration on MNIST.}
  \begin{center}
    \begin{small}
    	\begin{sc}
      \begin{tabular}{c | ccccc}
        \toprule
        \textbf{Method} & WGAN   & WGANGP & VAE   & WAE   & Ours   \\ \midrule
        \textbf{Time (ms)} & 26.17 & 47.03 & 7.38 & 7.22 & 11.08 \\
        \bottomrule
      \end{tabular}
  \end{sc}
    \end{small}
  \end{center}
\label{tab:time}
\end{table}

 \section{Discussion of Limitations}
\label{sec:discuss}

We have also run our method on the CelebA and CIFAR10
datasets \cite{celeba,CIFAR10}.  On CelebA, our method generates clear faces with good
visual quality and with meaningful latent space interpolation, as shown in
\Cref{fig:celeba_our} and \Cref{fig:celeba_interp}. However, we observe that
the good visual quality partly comes from the average face effect: the
expressions and backgrounds of generated images lack diversity compared with
GAN-based methods.

\begin{figure}[t]
	\centering {\
		\subfigure[Generated samples] {
			\label{fig:celeba_our}
			\includegraphics[width=0.3\columnwidth]{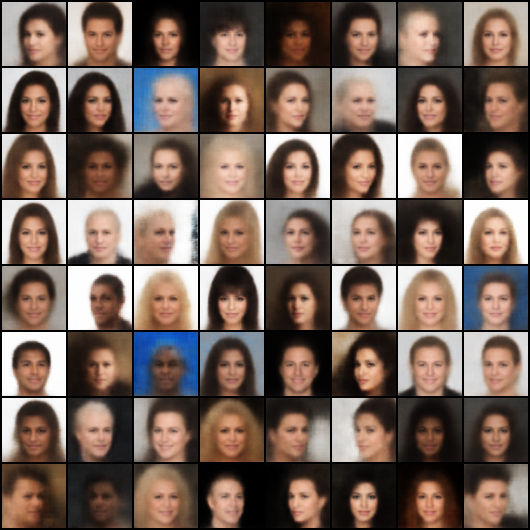}
		}
		\hfil
		\subfigure[Latent space walk] {
			\label{fig:celeba_interp}
			\includegraphics[width=0.3\columnwidth]{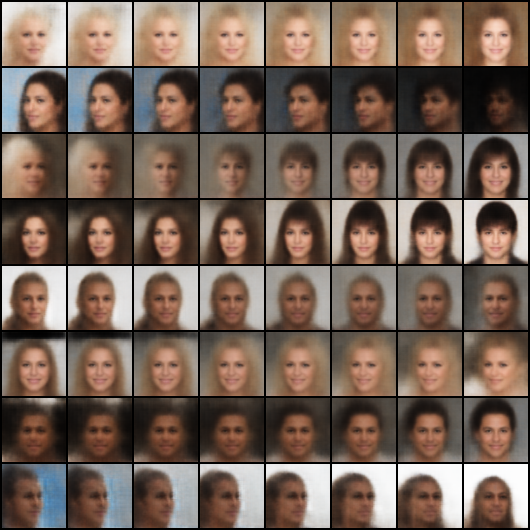}
		}
	}
	\caption{
          Samples generated by our method on CelebA, and a latent space
          interpolation.
	}
	\label{fig:celeba_qualitative}
\end{figure}

\Cref{fig:cifar10_qualitative} shows the results of our method on CIFAR10. As
shown, our method generates identifiable objects, but they are more blurry than
GAN-based models. VAE generates objects that are also blurry but less
identifiable.  We compute the Wasserstein-1 distance of the compared methods on
CIFAR10: our method (655), WGAN-GP (849) and VAE (745).  Our method achieves
the lowest Wasserstein distance but does not have better visual quality than
GAN-based models on CIFAR10.

\begin{figure}[t]
	\centering {\
		\subfigure[Our method] {
			\label{fig:cifar_our}
			\includegraphics[width=0.3\columnwidth]{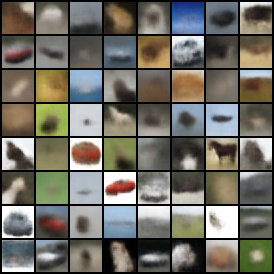}
		}
	\hfil
	\subfigure[VAE] {
		\label{fig:cifar_vae}
		\includegraphics[width=0.3\columnwidth]{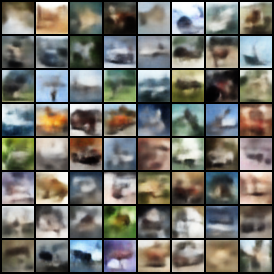}
	}
}
\caption{Samples generated by our method and VAE on CIFAR10.
	}
	\label{fig:cifar10_qualitative}
\end{figure}

Analyzing these results, we conjecture that minimizing Wasserstein distances on
pixel-wise metrics such as $\ell_1$ and $\ell_2$ leads to a mode-collapse-free
regularization effect. For models that minimize the Wasserstein distance, the
primary task inherently tends to cover all the modes disregarding the sharpness
of the generated samples. This is because not covering all the modes will
result in huge transport cost.  In GANs, the primary task is to generate sharp
images which can fool the discriminator, and some modes can be dropped towards
this goal. Consequently,  the objective of our model naturally prevents it from
mode collapse, but at the cost of generating more blurry samples. We propose
two potential remedies to the blurriness issue: one is to use a perceptual loss
\citep{GLO}; and the other is to incorporate adversarial metric into the
framework.  We leave them as future work.

 \section{Related Work}
\label{sec:related_work}

\paragraph{Optimal Transport.} Optimal transport is an old yet vibrant topic \citep{villani_1,
  villani_2, cuturi_book}. \citet{cuturi_stochastic_ot} give the stochastic formulation of
semi-discrete optimal transport used in our solver. \citet{geometric_ot}
provide a geometric view of semi-discrete optimal transport. \citet{seguy}
propose to parameterize the Kantorovich potential via neural networks. They also propose to train
generative networks by fitting towards the optimal transport plan between latent
code and data, which can be considered as a special case of the non-alternating
procedure we discussed earlier.

\paragraph{Generative models and OT.} Combining generative modeling and optimal transport has been studied extensively.
One line of research comes from the dual representation of Wasserstein
distance via the Kantorovich-Rubinstein equality \citep{WGAN, WGANGP}.  Another
line optimizes the primal form of Wasserstein distance by relaxing it to a
penalized form \citep{VEGAN, WAE}. \citet{cuturi_GAN_VAE} give a 
comparison of WGAN and WAE from the view of optimal transport. 
By contrast, our work evaluate Wasserstein distance in dual space, and then optimizes it using its primal form.

The idea of computing optimal transport between batches of generated and real
samples has been used in both non-adversarial generative modeling \citep{sinkhorn_autodiff, xie},
as well as adversarial generative modeling \citep{OTGAN, WGANTS}. However, minimizing
batch-to-batch transport distance does not lead to the minimization of the
Wasserstein distance between the generated and target distributions
\citep{cramergan}.  Instead, our method computes the whole-to-whole optimal
transport via the semi-discrete formulation.

\paragraph{Generalization properties of Wasserstein distance.}
\citet{SriperumbudurFGSL2012} analyze the sample complexity of evaluating integral probability metrics.
\citet{generalization_equilibrium} show that KL-divergence and Wasserstein
distances do not generalize well in high dimensions, but their neural net distance
counterparts do generalize. \citet{thin8} give reasons for the advantage of GAN
over VAE, and collects the Thin-8 dataset to demonstrate the advantage of GANs,
which is used in our experiments. In this work we have compared both our theoretical
and empirical findings with those in \citep{generalization_equilibrium, thin8}.
 \section{Conclusion and Future Work}
\label{sec:conclusion}

We have proposed a simple alternating procedure to generative modeling by
explicitly optimizing the Wasserstein distance between the generated distribution
and real data. We show theoretically and empirically that our method does
optimize Wasserstein distance to the training dataset, and generalizes to the
underlying distribution.  Interesting future work includes combining our method
with adversarial and perceptual losses, and theoretical analysis on how gradual
fitting contributes to smoother manifolds and better generalization.

\subsection*{Acknowledgements}

BB and MJT are grateful for support from the NSF under grant IIS-1750051.
DH acknowledges support from the Sloan Foundation. JP acknowledges support from Tencent AI Lab Rhino-Bird Gift Fund.
 
\bibliography{bib}
\bibliographystyle{plainnat}

\begin{appendices}
      \section{Proof of \Cref{prop:argmin}}

\begin{proof}
	By \citep[Theorem 2.44]{villani_1}, if $T(x)$ can be uniquely determined by $T(x) = x - \nabla c^*(\nabla\varphi(x))$, where $\varphi$ is defined in \cref{eq:kantorovich_duality}, then $T(x)$ is the unique Monge-Kantorovich optimal transference plan.
        Defining $m := \argmin_i c(x - y_i) - \hpsi_i$ for convenience, we have
	\begin{align*}
	T(x) &= x - \nabla c^*(\nabla\varphi(x))\\
	\Longrightarrow\qquad x - T(x) 
	&= \nabla c^*(\nabla\varphi(x))\\
	&= \nabla c^*(\nabla_x \min_i c(x - y_i) - \hpsi_i)\\
	&= \nabla c^* \nabla_x (c( x-y_m)-\hpsi_m)\\
	&= \nabla c^* \nabla c(x-y_m)\\
	&=x - y_m\\
	\Longrightarrow\qquad T(x) &= y_m\\
	&=y_{\argmin_i c(x - y_i)-\hpsi_i}.
	\end{align*}
\end{proof}

\section{An Example of Non-Gradual Training}

We present in \Cref{fig:gradual_bad} a similar synthetic example to \Cref{fig:gradual:2}, except that FIT step will iterate itself until stuck at a local optimum. The local optimum after the first OTS-FIT run is a zig-zag shaped curve which does not generalize. While alternating the OTS-FIT is able to further push the generated samples to targets, the learned manifold is not as smooth as in \Cref{fig:gradual:2}.

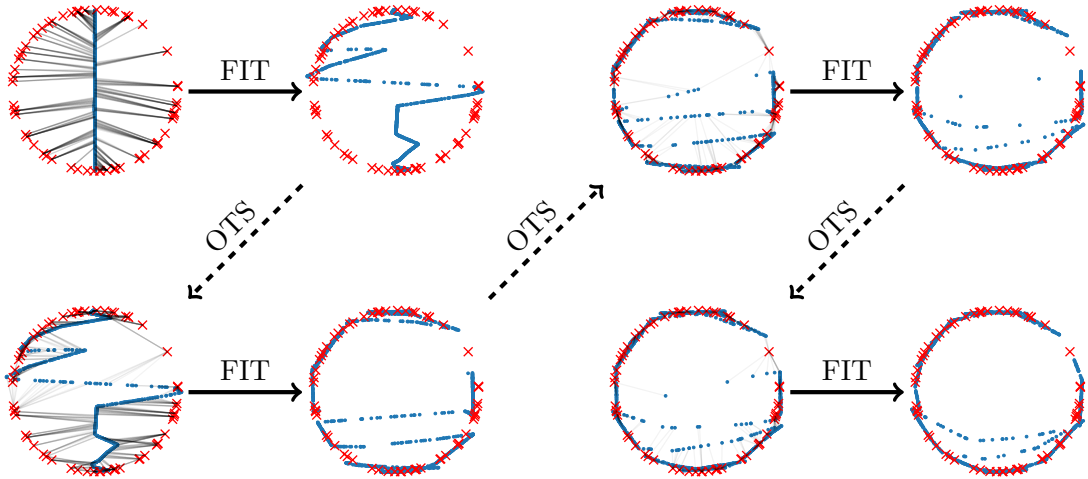
\begin{figure}[h]
	\centering
\begin{tikzpicture}
	\tikzset{node style/.style={midway,above,sloped}};
	\node[inner sep=0pt] (f1) at (0,0)
	{\includegraphics[width=.15\textwidth]{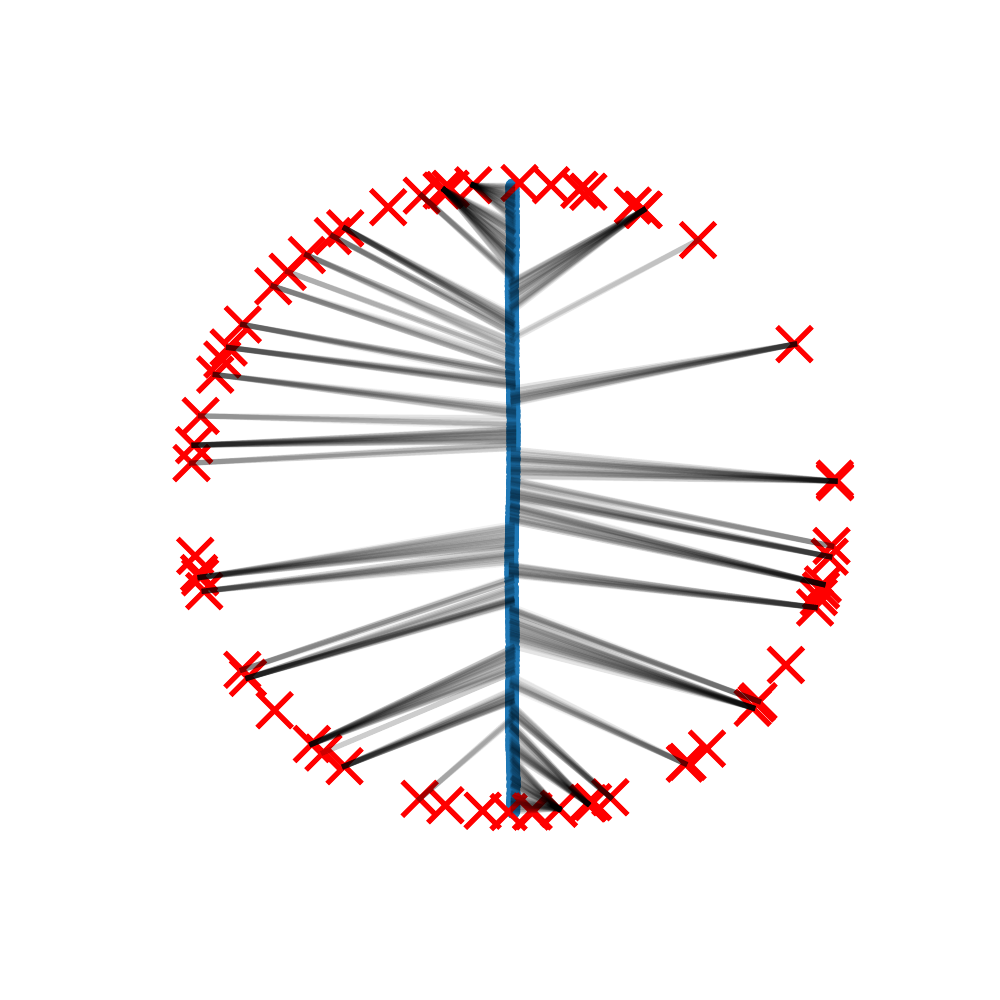}};
	\node[inner sep=0pt] (f2) at (4,0)
	{\includegraphics[width=.15\textwidth]{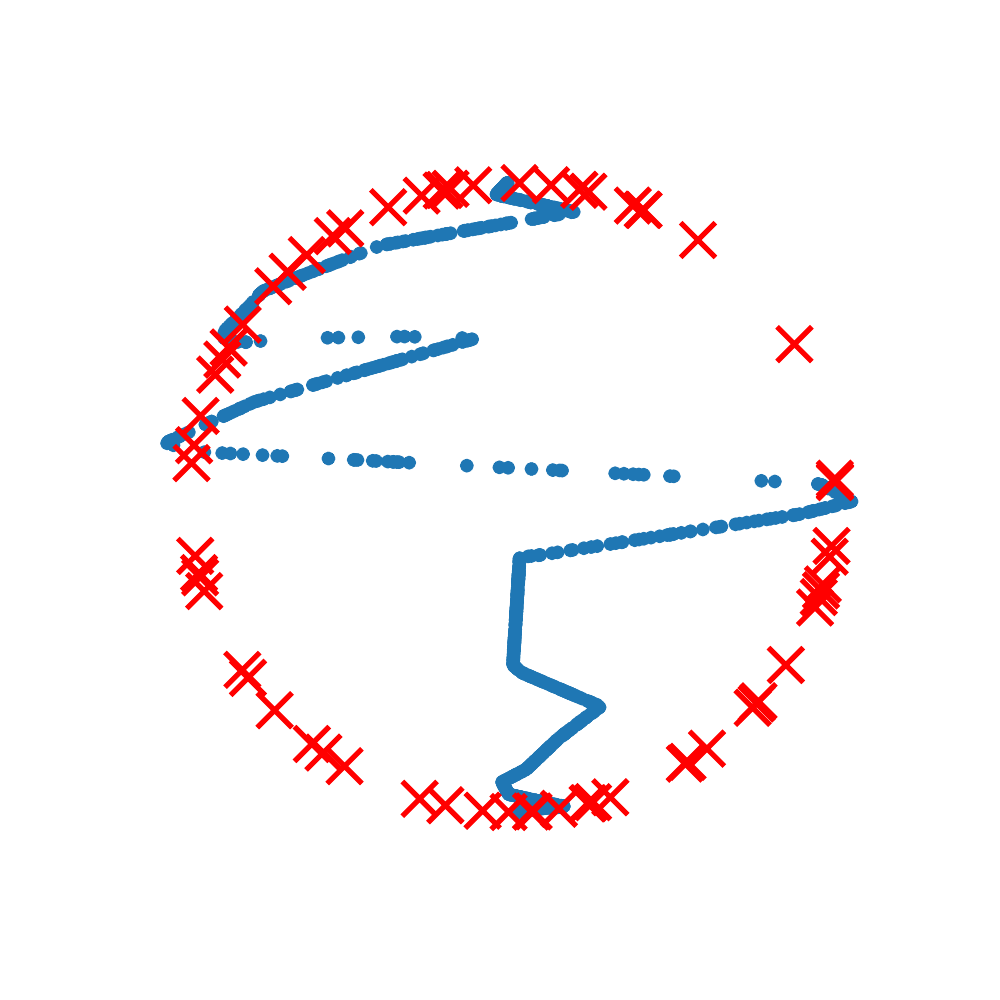}};
	\node[inner sep=0pt] (f3) at (0,-4)
	{\includegraphics[width=.15\textwidth]{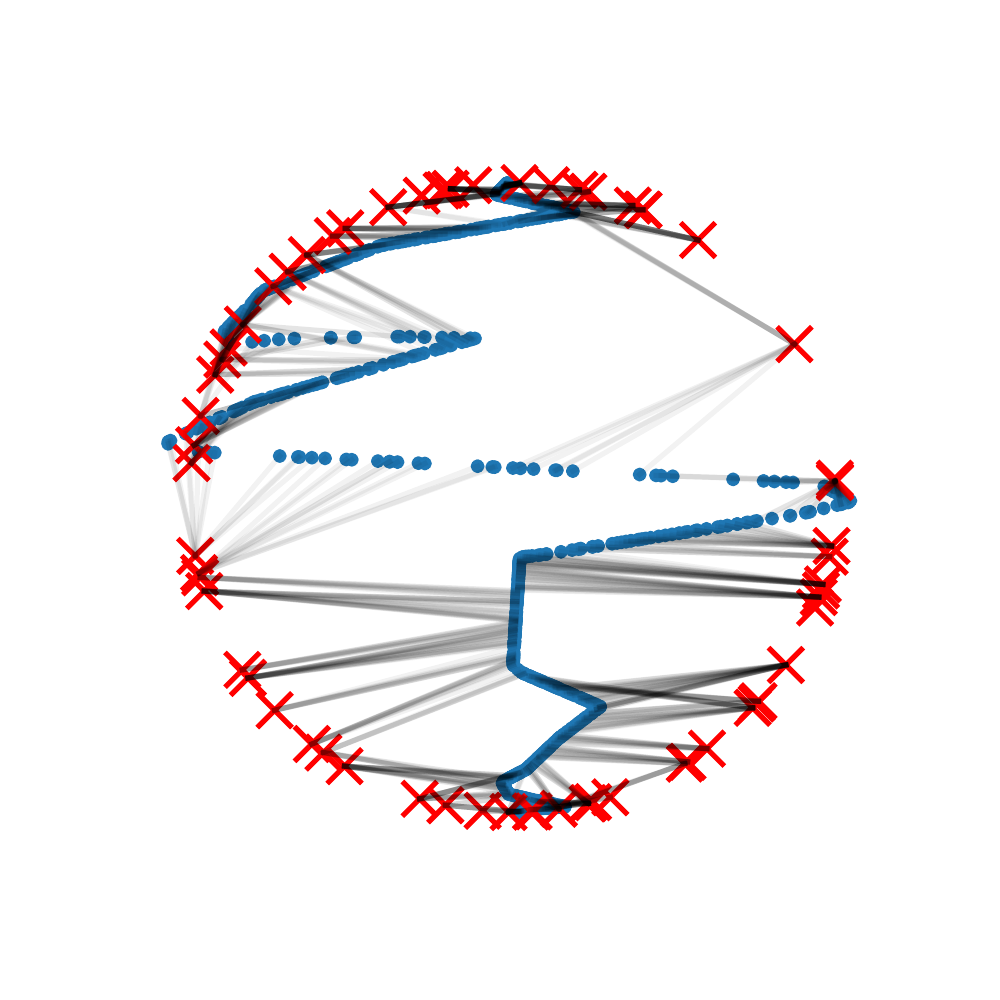}};
	\node[inner sep=0pt] (f4) at (4,-4)
	{\includegraphics[width=.15\textwidth]{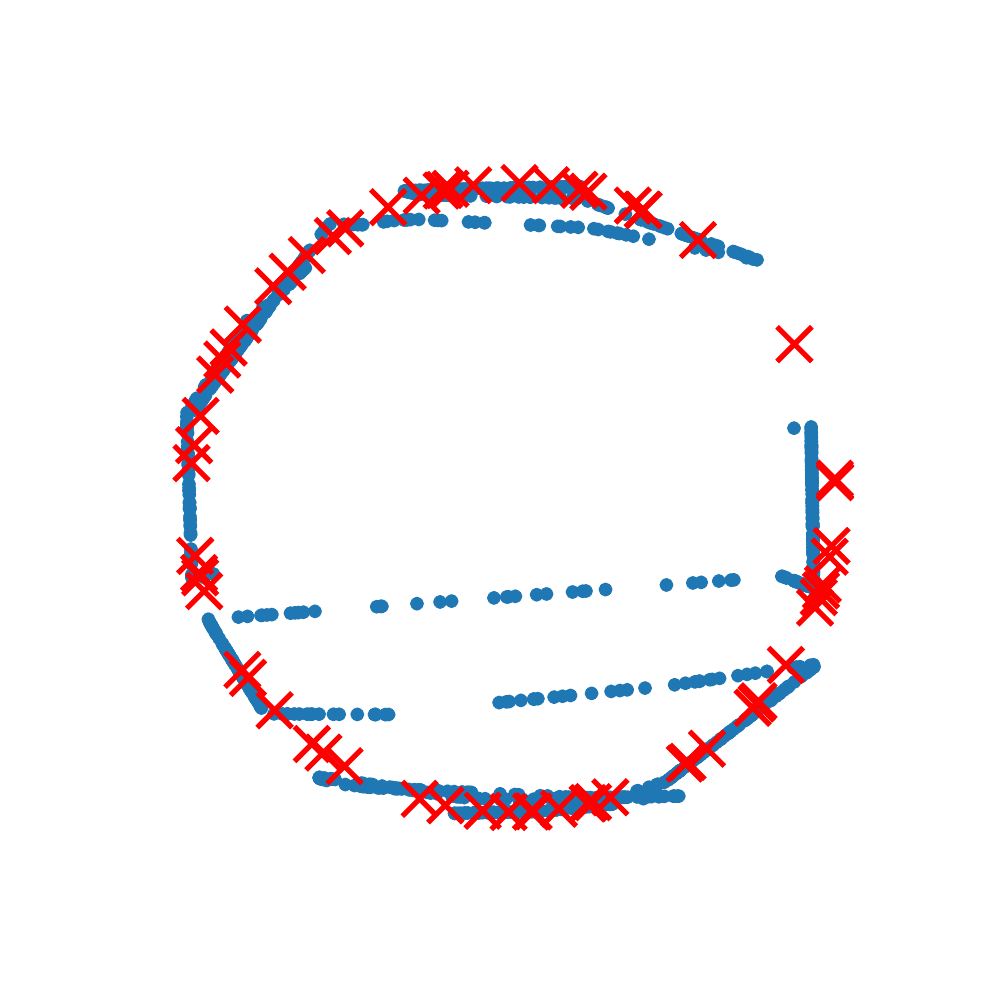}};
	\node[inner sep=0pt] (f5) at (8,0)
	{\includegraphics[width=.15\textwidth]{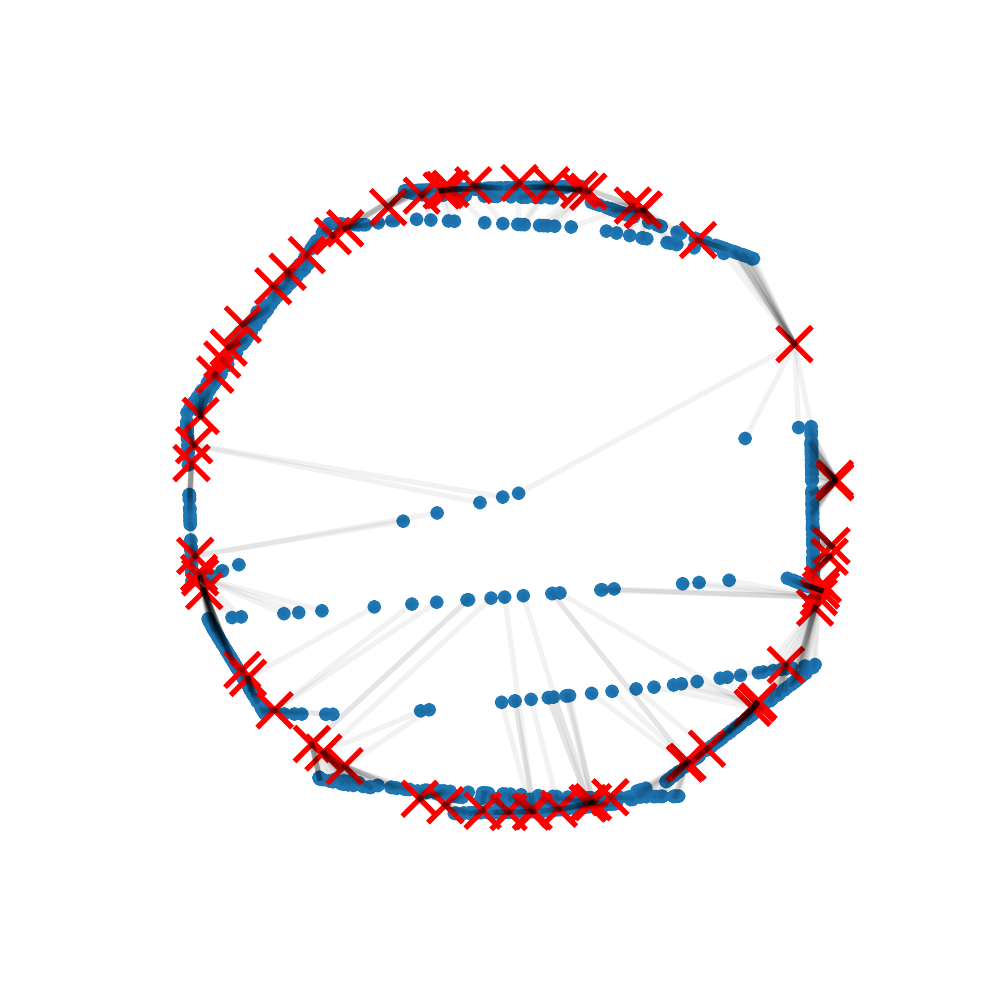}};
	\node[inner sep=0pt] (f6) at (12,0)
	{\includegraphics[width=.15\textwidth]{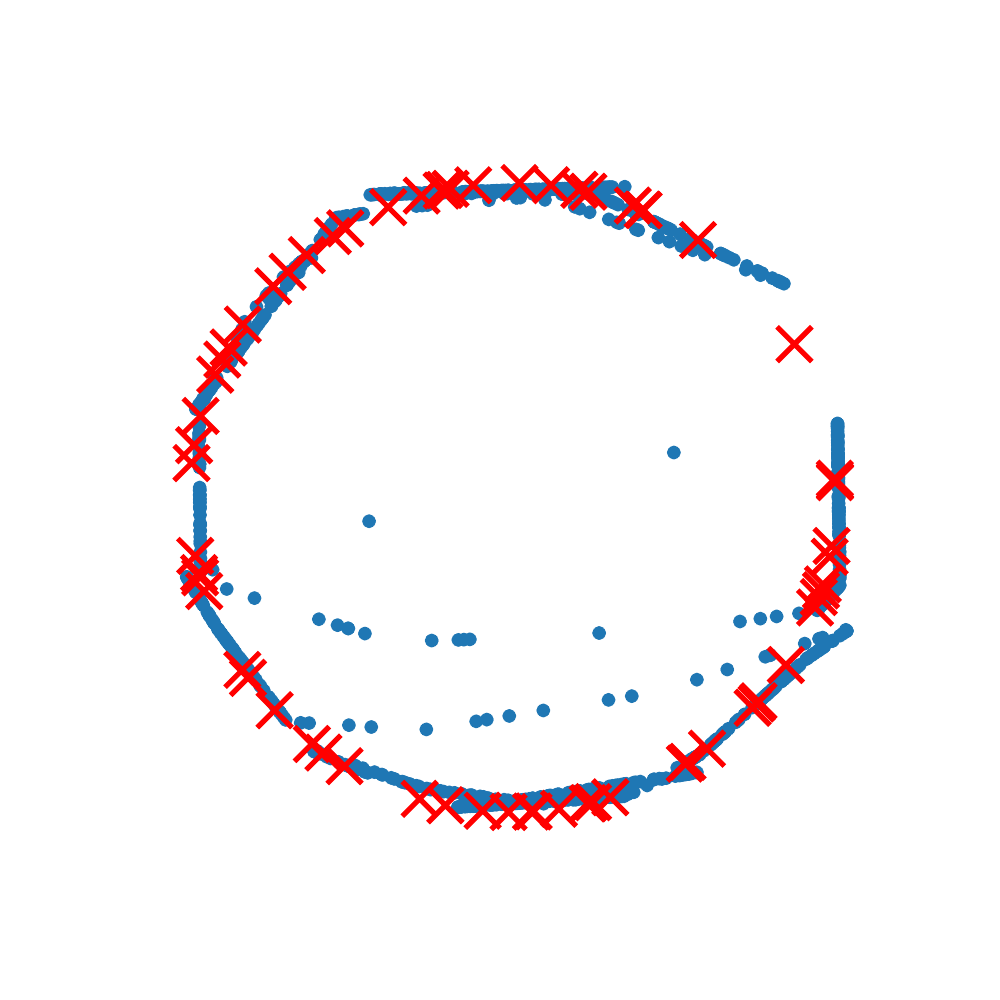}};
	\node[inner sep=0pt] (f7) at (8,-4)
	{\includegraphics[width=.15\textwidth]{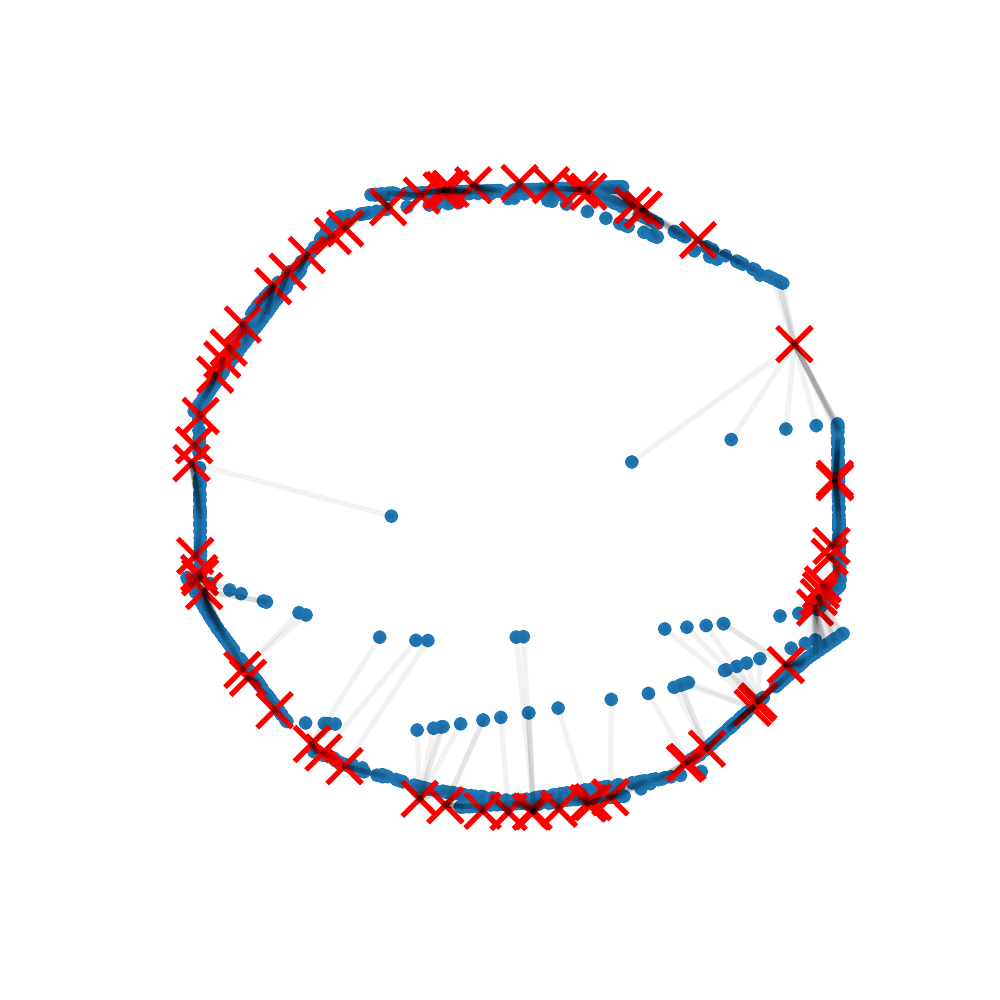}};
	\node[inner sep=0pt] (f8) at (12,-4)
	{\includegraphics[width=.15\textwidth]{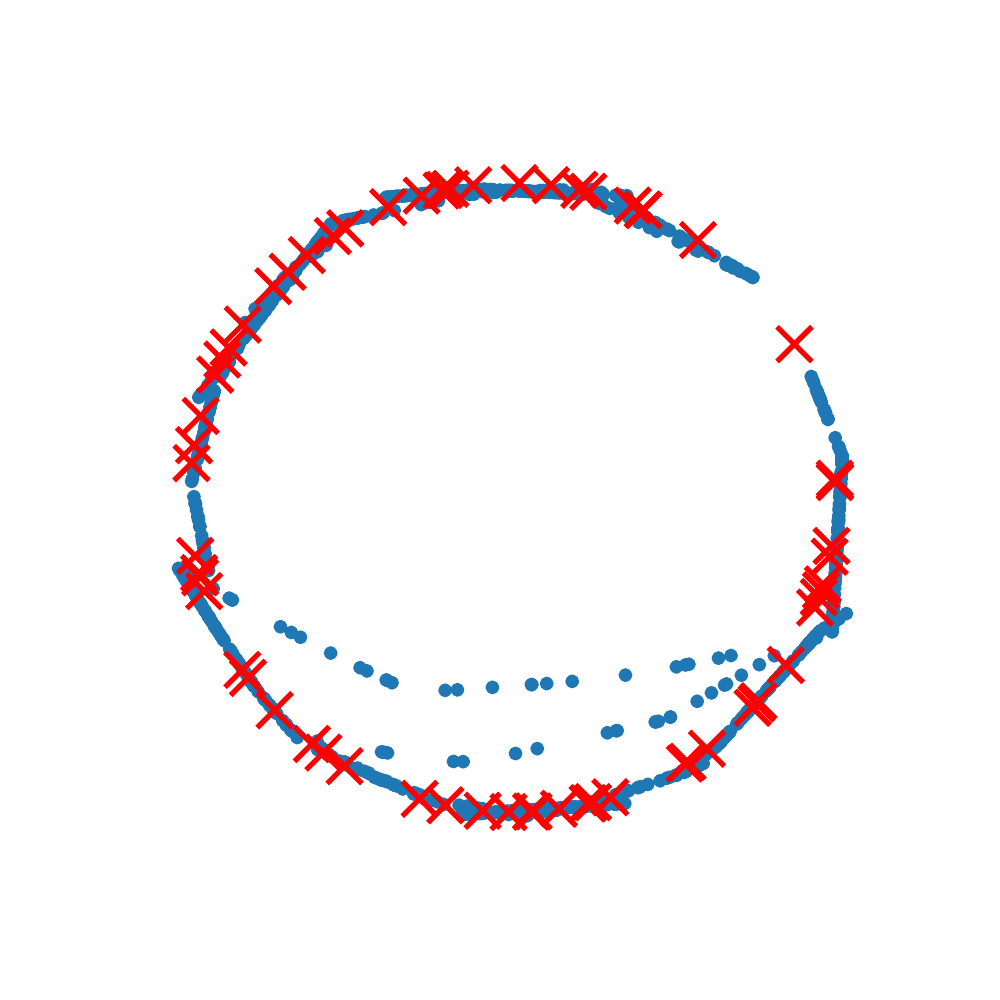}};
	\draw[->,ultra thick] (f1.east) -- (f2.west)
	node[node style] {FIT};
	\draw[->,dashed, ultra thick] (f2.south west) -- (f3.north east)
	node[node style] {OTS};
	\draw[->,ultra thick] (f3.east) -- (f4.west)
	node[node style,] {FIT};
	\draw[->,dashed, ultra thick] (f4.north east) -- (f5.south west)
	node[node style,] {OTS};
	\draw[->,ultra thick] (f5.east) -- (f6.west)
	node[node style,] {FIT};
	\draw[->,dashed, ultra thick] (f6.south west) -- (f7.north east)
	node[node style,] {OTS};
	\draw[->,ultra thick] (f7.east) -- (f8.west)
	node[node style,] {FIT};
	\end{tikzpicture}
	\caption{Example of non-gradual training.}
	\label{fig:gradual_bad}
\end{figure}

\section{Further Discussion on the Uniqueness of OT Plan}

Continuity of $g\#\mu$ is required in \citep[Theorem 2.44]{villani_1}, which is violated if $g\#\mu$ lies in a low-dimensional manifold in $\R^d$, or the activation function is not strictly increasing (such as ReLU). When $g\#\mu$ is not continuous, the optimal transference plan is not unique in general, and moreover it is possible that $T(x):=y_{\argmin_i c(x,y_i)-\hpsi_i}$ is not a transference plan at all (\cref{fig:uniqueness_a}).

However, as discussed in \Cref{prop:argmin}, the only condition in which the ``argmin form'' does not characterize a unique Monge OT plan is the existence of ties. To break the ties, one can add arbitrarily small $d$-dimensional Gaussian noise to the outputs of $g$ (\cref{fig:uniqueness_b}).  Furthermore, this makes the distribution continuous and satisfies the conditions of \citep[Theorem 2.44]{villani_1}. On the other hand, supposing that the dataset $\hnu$ is drawn from a continuous distribution in $\R^d$, the event that two samples achieves exactly the same minimum $\min_i c(x,y_i)-\hpsi_i$ happens with probability zero (\cref{fig:uniqueness_c}).

\begin{figure}[h]
	\centering {\
		\subfigure[Example of non-uniqueness.] {
			\label{fig:uniqueness_a}
			\includegraphics[width=0.3\columnwidth]{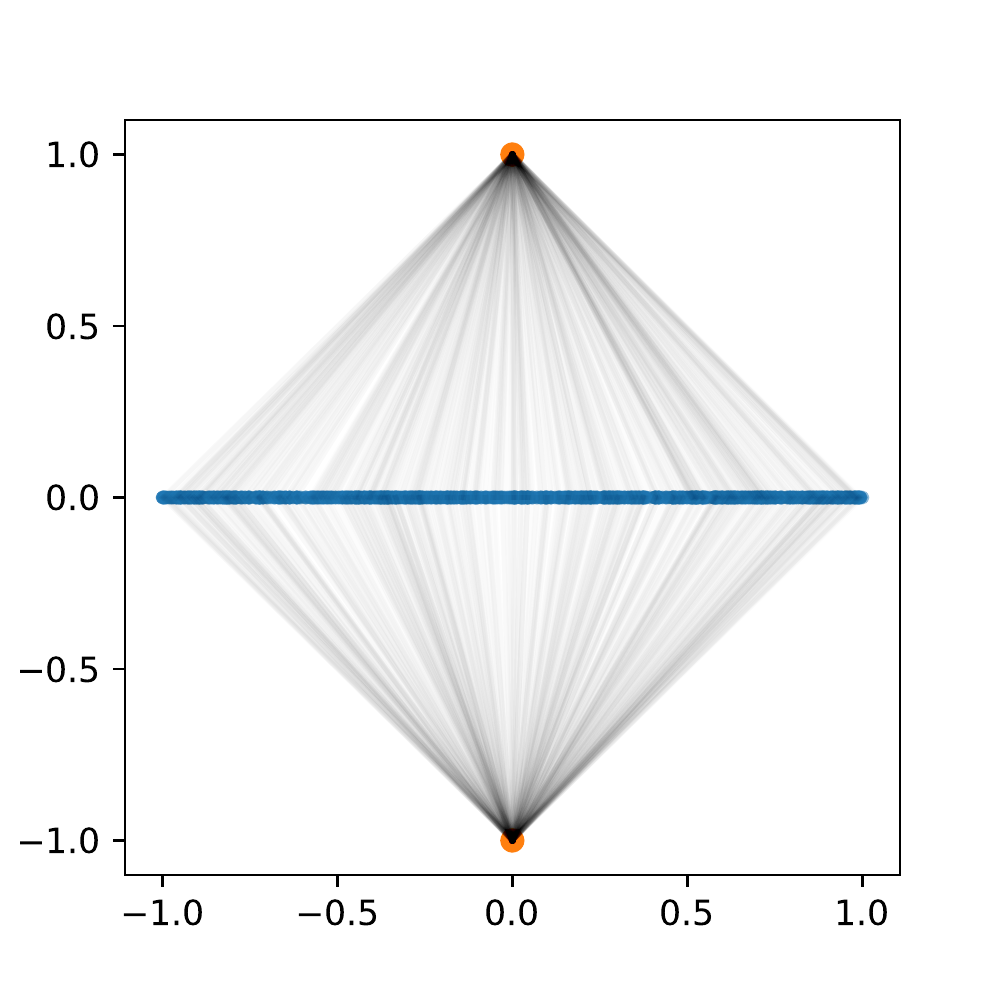}
		}
	}
	\centering {\
		\subfigure[Any perturbation breaks the ties.] {
			\label{fig:uniqueness_b}
			\includegraphics[width=0.3\columnwidth]{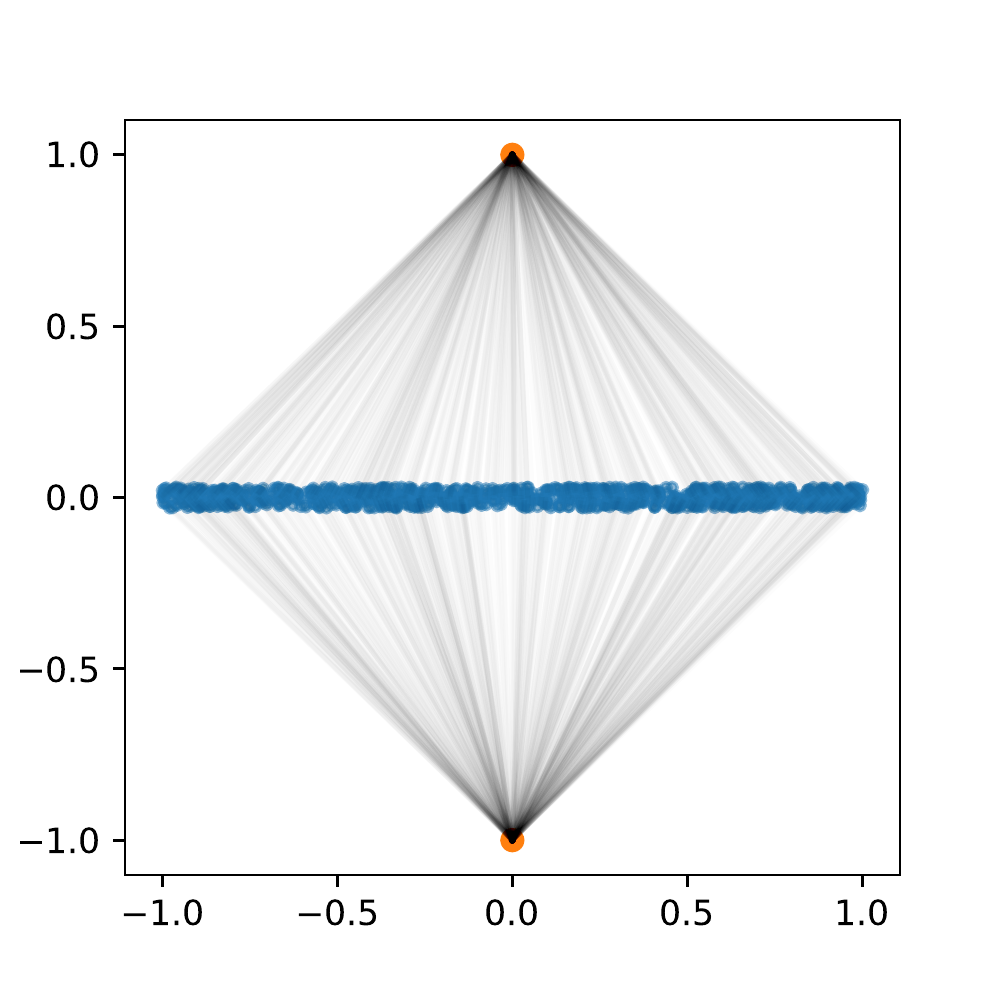}
		}
	}
	\centering {\
		\subfigure[When the samples are shifted.] {
			\label{fig:uniqueness_c}
			\includegraphics[width=0.3\columnwidth]{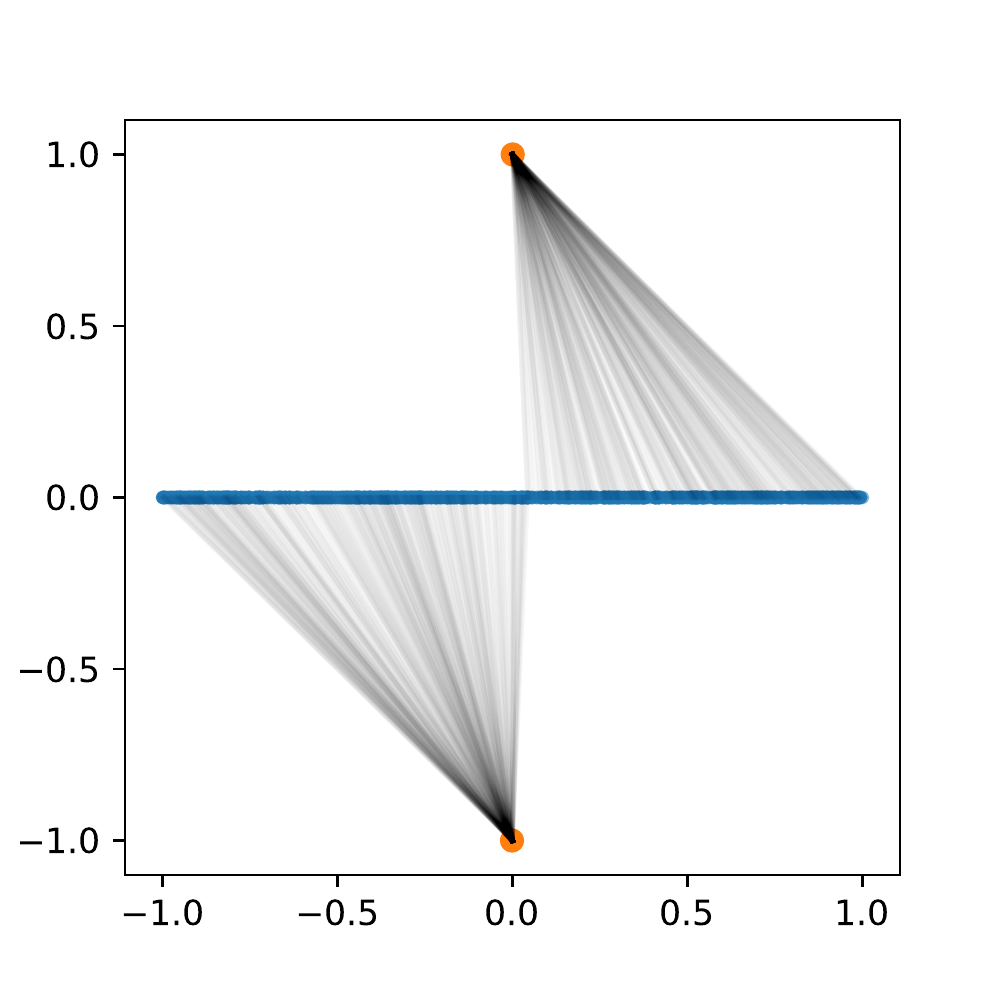}
		}
	}
	\caption{(a) Suppose $g\#\mu$ is a uniform distribution on $[-1,1] \times \{0\}$, and the samples are $y_1(0, 1), y_2(0,-1)$, then any equal-sized partition of $[-1,1]$ would result to an optimal Monge plan. (b) Adding a small perturbation would make the optimal Monge plan unique. (c) If the samples become $y_1(0.001, 1), y_2(-0.001,-1)$, the ties are also broken.
	}
	\label{fig:uniqueness}
\end{figure}

\section{Verifying the Assumption in \Cref{sect:generalization}}

We train $\hpsi$ between our fitted generating distribution $g\#\mu$ and MNIST full dataset $\hnu$, then fit an MLP $\psi$ with 4 hidden layers of 512 neurons to $\hpsi$ on training dataset $\hnu_\text{train}$, and evaluated on both $\hnu_\text{train}$ and test dataset $\hnu_\text{test}$ ($\psi$ takes 784-dimensional images as input and outputs a scalar). \Cref{fig:verify} shows that the training error goes to zero and $\psi$ has almost the same value as $\hpsi$ when evaluated on $\hnu$.

\begin{figure}[h]
	\centering {\
		\subfigure[L2 loss during training] {
			\label{fig:verify_a}
			\includegraphics[width=0.3\columnwidth]{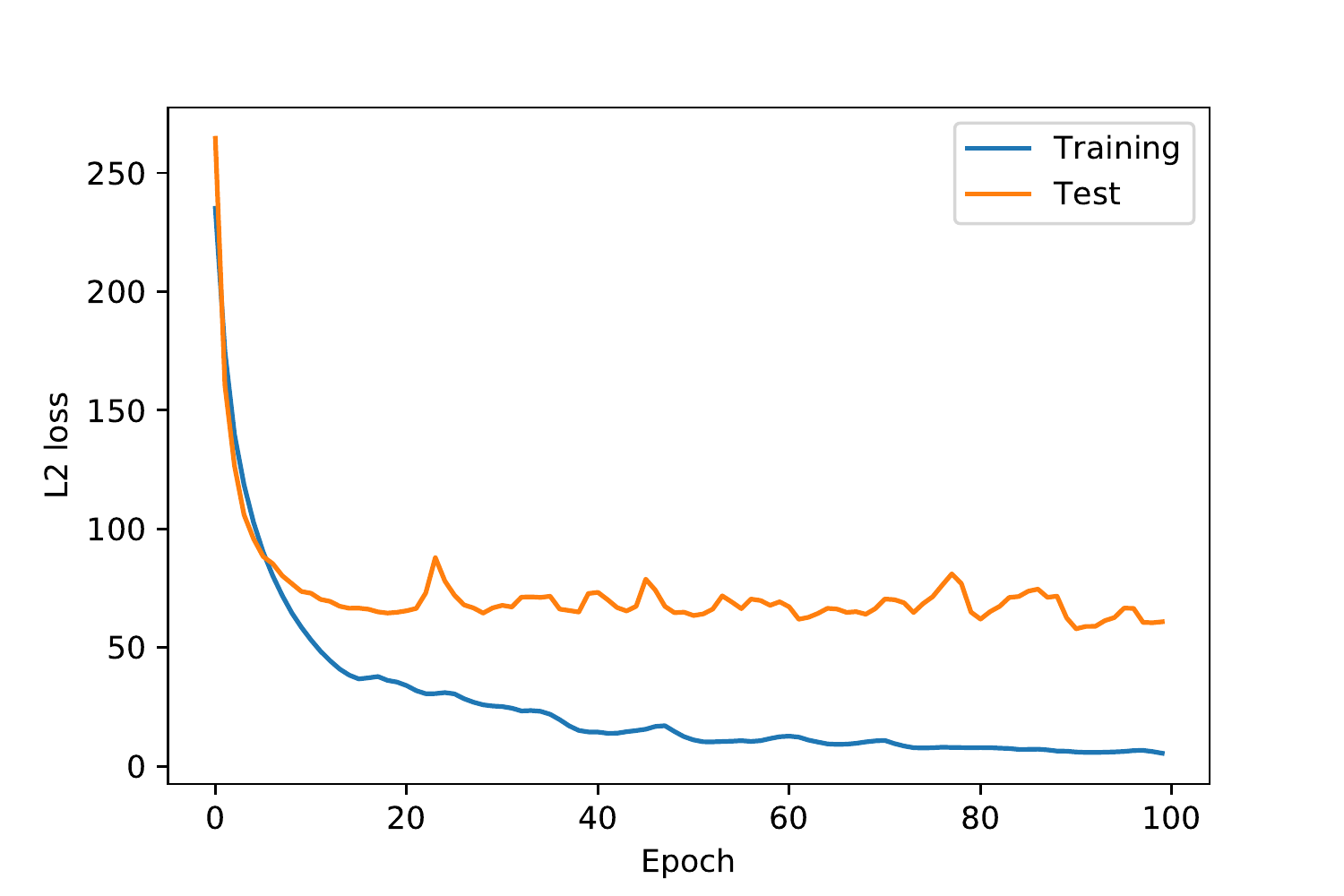}
		}
	}
	\centering {\
		\subfigure[$\hpsi$ vs fitted $\psi$ on 100 training samples] {
			\label{fig:verify_b}
			\includegraphics[width=0.3\columnwidth]{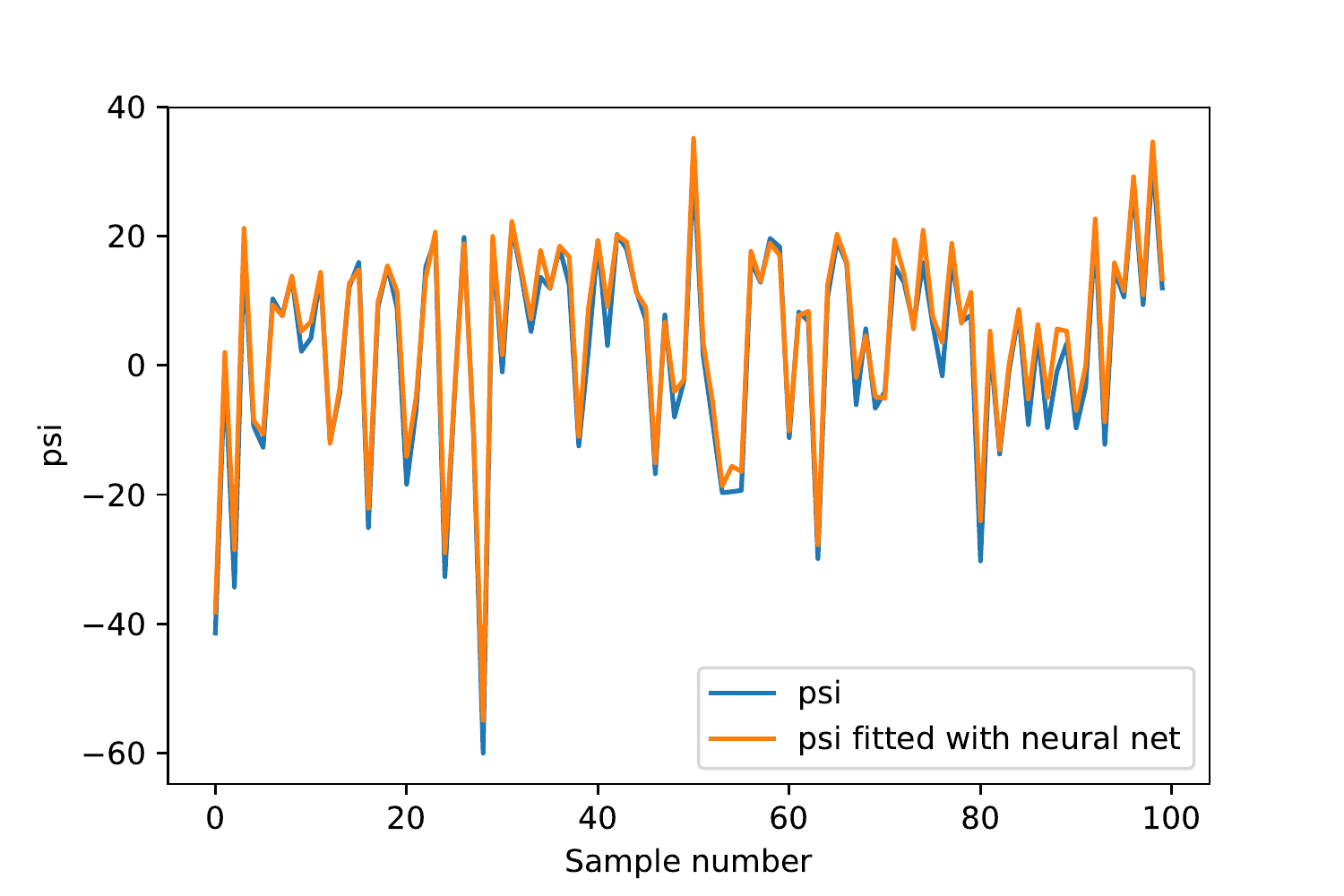}
		}
	}
	\centering {\
		\subfigure[$\hpsi$ vs fitted $\psi$ on 100 test samples] {
			\label{fig:verify_c}
			\includegraphics[width=0.3\columnwidth]{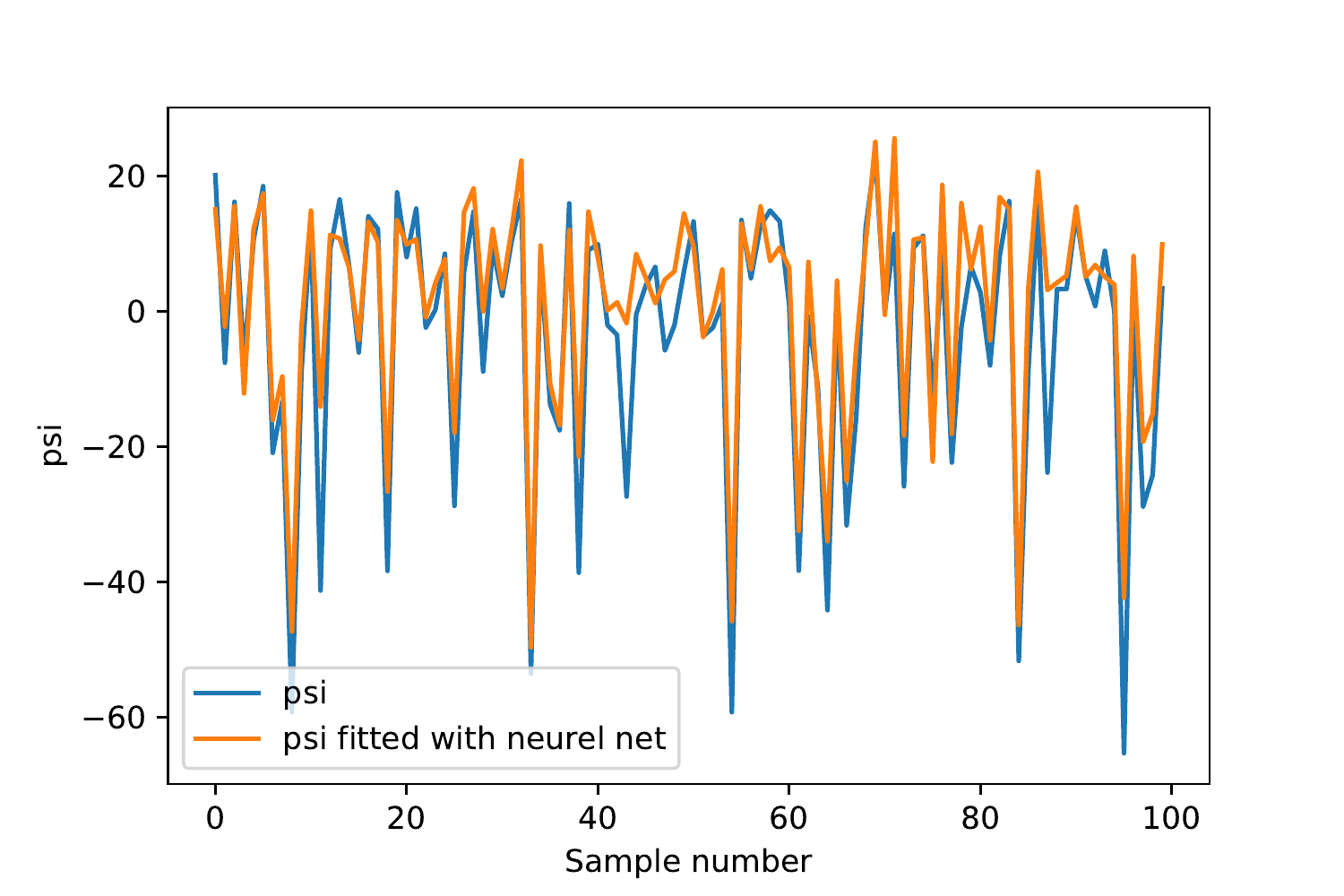}
		}
	}
	\caption{$\hpsi$ on MNIST fitted by neural network.
	}
	\label{fig:verify}
\end{figure}

We point out some concrete cases where the exponential lower bound is defeated by introducing a dependence on the measures $\mu$ and $\nu$, without such empirical verification. If $\mu$ and $\nu$ are both piecewise constant densities with few pieces, then the supremum in \Cref{eq:cond:apx} can be well-approximated with piecewise affine functions with few pieces. In the worst case, the number of pieces can be exponential in dimension, but we will depend explicitly on the measure and its properties.

Another way to achieve the generalization result without the verification process, is to parametrize $\psi$ as a neural network in the optimal transport algorithm, an idea related to \citep{seguy}. In \Cref{alg:ots}, we choose to optimize the vectorized $\hpsi$ since it is a convex programming formulation guaranteed to converge to global minimum.

\section{Experimental Details}

\subsection{Evaluation Metrics}
\begin{description}
	\item[\textbf{Neural net distance (NND-WC, NND-GP).}]
	\cite{generalization_equilibrium} define the \emph{neural net distance}
	between the generated distribution $g\#\mu$ and dataset $\nu$ as:
	$$\mathcal{D}_{\mathcal{F}}(g\#\mu, \nu) :=
	\sup_{f\in\mathcal{F}}\mathbb{E}_{x\sim g\#\mu}f(x) -
	\mathbb{E}_{y\sim\nu}f(y),$$ where $\mathcal{F}$ is a neural network
	function class. We use DCGAN with weight clipping at $\pm 0.01$, and DCGAN
	with gradient penalty with $\lambda=10$ as two choices of $\mathcal{F}$. We
	call the corresponding neural net distances NND-WC and NND-GP respectively.
	
	\item[\textbf{Wasserstein-1 distance (WD).}] This refers to the exact Wasserstein distance
	on $\ell_1$ metric between the generated distribution $\mu$ and dataset
	$\nu$, computed with our \Cref{alg:ots}.
	
	\item[\textbf{Inception score (IS).}] \cite{improvegan} assume there exists a
	pretrained external classifier outputing label distribution $p(y|x)$ given
	sample $x$. The score is defined as $\text{IS}_p(\mu) :=
	\exp\{\mathbb{E}_{x\sim\mu}\text{KL}(p(y|x) \;||\; p(y))\}$. 
	
	\item[\textbf{Fr\'echet Inception distance (FID).}] \citet{FID} give this
	improvement over IS, which compares generated and real samples by the
	activations of a certain layer in a pretrained classifier. Assuming the
	activations follow Multivariate Gaussian distribution of mean $\mu_g,\mu_r$
	and covariance $\Sigma_g, \Sigma_r$, FID is defined as: 
	$$
	\text{FID}(\mu_g, \mu_r, \Sigma_g, \Sigma_r)\\
	:= \|\mu_g-\mu_r\|^2 + \text{Tr}(\Sigma_r + \Sigma_g - 2(\Sigma_r\Sigma_g)^{1/2}).
	$$
\end{description}

Note that the naming of IS and FID is used since the classifier is an Inception-v3
network pretrained on Imagenet. Since the pretrained Inception network is not
suitable for classifying handwritten digits, we follow an idea due to
\citet{li2017alice} and pretrain a CNN classifier on MNIST dataset with 99.1\%
test accuracy.

\subsection{Neural Network Architectures and Hyperparameters}

\subsubsection{Generator}

\paragraph{MLP:}
$$\textsl{Input} (100) \xrightarrow{FC} \textsl{Hidden}(512) \xrightarrow{FC} \textsl{Hidden}(512)\xrightarrow{FC} \textsl{Hidden}(512)\xrightarrow{FC} \textsl{Output}(D). $$

\paragraph{DCGAN MNIST $28\times28$:}
\begin{align*}
&\textsl{Input} (100, 1, 1) \xrightarrow{\textsl{ConvT}(7,1,0)+\textsl{BN}} \textsl{Hidden}(128, 7, 7) \xrightarrow{\textsl{ConvT}(4,2,1)+\textsl{BN}} \textsl{Hidden}(64, 14, 14)\\
&\xrightarrow{\textsl{ConvT}(4,2,1)} \textsl{Output}(1, 28, 28).
\end{align*}

\paragraph{DCGAN Thin-8 $128\times128$:}
\begin{align*}
&\textsl{Input} (100, 1, 1) \xrightarrow{\textsl{ConvT}(4,1,0)+\textsl{BN}} \textsl{Hidden}(256, 4, 4) \xrightarrow{\textsl{ConvT}(4,2,1)+\textsl{BN}}\textsl{Hidden}(128, 8, 8)\\
&\xrightarrow{\textsl{ConvT}(4,2,1)+\textsl{BN}}\textsl{Hidden}(64, 16, 16)
\xrightarrow{\textsl{ConvT}(4,2,1)+\textsl{BN}} \textsl{Hidden}(32, 32, 32)\xrightarrow{\textsl{ConvT}(4,2,1)} \textsl{Hidden}(16, 64, 64)\\
&\xrightarrow{\textsl{ConvT}}\textsl{Output}(1, 128, 128).
\end{align*}

\subsubsection{Discriminator/Encoder}

\paragraph{DCGAN Discriminator $28 \times 28$:}
\begin{align*}
\textsl{Input} (1, 28, 28) \xrightarrow{Conv(4,2,1)} \textsl{Hidden}(64, 14, 14) \xrightarrow{\textsl{ConvT}(4,2,1)+\textsl{BN}}\textsl{Hidden}(128, 7, 7)
\xrightarrow{FC} \textsl{Output}(1).
\end{align*}

\paragraph{Thin-8 Discriminator $128 \times 128$:}
\begin{align*}
&\textsl{Input} (1, 128, 128) \xrightarrow{Conv(4,2,1)} \textsl{Hidden}(16, 64, 64) \xrightarrow{\textsl{ConvT}(4,2,1)+\textsl{BN}}\textsl{Hidden}(32, 32, 32)\\&
\xrightarrow{\textsl{ConvT}(4,2,1)+\textsl{BN}}\textsl{Hidden}(64, 16, 16)
\xrightarrow{\textsl{ConvT}(4,2,1)+\textsl{BN}}\textsl{Hidden}(128, 8, 8)
\xrightarrow{\textsl{ConvT}(4,2,1)+\textsl{BN}}\textsl{Hidden}(256, 4, 4)\\&
\xrightarrow{FC} \textsl{Output}(1).
\end{align*}

Encoder architectures are the same as discriminator architectures except for the output layers.

For WGANGP, batch normalization is not used since it affects computation of gradients. All activations used are ReLU. Learning rate is $10^{-4}$ for WGAN, WGANGP and our method, and $10^{-3}$ for VAE and WAE.

\subsection{Training Details of Our Method}

To get the transport mapping $T$ in OTS, we memorize the sampled batches and their transportation targets, and reuse these batches in FIT. By this trick we avoid recomputing the maximum over the whole dataset.

Our empirical stopping criterion relies upon keeping a histogram of transportation targets in memory:
if the histogram of targets is close to a uniform distribution (which is the distribution of training dataset), we stop OTS.
This stopping criterion is grounded by our analysis in \Cref{sec:optimization}.
 \end{appendices}

\end{document}